\documentclass{article}

\usepackage{arxiv}

\usepackage[utf8]{inputenc} % allow utf-8 input
\usepackage[T1]{fontenc}    % use 8-bit T1 fonts
\usepackage{hyperref}       % hyperlinks
\usepackage{url}            % simple URL typesetting
\usepackage{booktabs}       % professional-quality tables
\usepackage{amsfonts}       % blackboard math symbols
\usepackage{nicefrac}       % compact symbols for 1/2, etc.
\usepackage{microtype}      % microtypography
\usepackage{lipsum}		% Can be removed after putting your text content
\usepackage{graphicx}
\usepackage{natbib}
\usepackage{doi}
\usepackage{amsmath}
\usepackage{times}
\usepackage{colortbl}
\usepackage{diagbox}
\usepackage{amsthm}
\usepackage{amssymb}
\usepackage{shuffle}
\usepackage{color}
\usepackage{multirow}
\usepackage{wrapfig}
\usepackage{anyfontsize}

\usepackage{pifont}% http://ctan.org/pkg/pifont
\newtheorem{theorem}{Theorem}[section]

\newtheorem{lemma}{Lemma}[section]
\newtheorem{definition}{Definition}[section]

\newtheorem{model}{Model}[section]
\newtheorem{remark}{Remark}[section]
\title{Logsig-RNN: a novel network for robust and efficient skeleton-based action recognition}

\author{ Shujian Liao\\
	Department of Mathematics\\
	University College London\\
	\texttt{shujian.liao.18@ucl.ac.uk} \\
	\And
Terry Lyons \\
	Mathematical Institute\\
	University of Oxford\\
	\texttt{terry.lyons@maths.ox.ac.uk} \\
	 \And
Weixin Yang \\
	Mathematical Institute\\
	University of Oxford\\
	\texttt{yangw2@maths.ox.ac.uk} \\
	\And
	\href{https://orcid.org/0000-0002-7660-1051}{\includegraphics[scale=0.06]{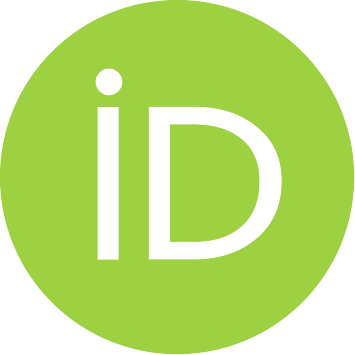}}
	Kevin Schlegel\\
	Department of Mathematics\\
	University College London\\
	\texttt{k.schlegel@ucl.ac.uk} \\
	\And
	\href{https://orcid.org/0000-0001-5485-4376}{\includegraphics[scale=0.06]{orcid.pdf}\hspace{1mm}}Hao Ni\thanks{Corresponding author.} \\
	Department of Mathematics\\
	University College London\\
	\texttt{h.ni@ucl.ac.uk} \\
}

\renewcommand{\shorttitle}{\textit{
Logsig-RNN for robust action recognition}}

\hypersetup{
pdftitle={Logsig-RNN: a novel network for robust and efficient skeleton-based action recognition},
pdfsubject={cs.LG},
pdfauthor={Shujian Liao, Terry Lyons, Weixin Yang, Kevin Schlegel, Hao Ni},
pdfkeywords={Skeleton-based action recognition, Recurrent neural network, Deep learning, Log-signature, Rough path theory },
}

\begin{document}
\maketitle

\begin{abstract}
	This paper contributes to the challenge of skeleton-based human action recognition in videos. The key step is to develop a generic network architecture to extract discriminative features for the spatio-temporal skeleton data. In this paper, we propose a novel module, namely Logsig-RNN, which is the combination of the log-signature layer and recurrent type neural networks (RNNs). The former one comes from the mathematically principled technology of signatures and log-signatures as representations for streamed data, which can manage high sample rate streams, non-uniform sampling and time series of variable length. It serves as an enhancement of the recurrent layer, which can be conveniently plugged into neural networks. Besides we propose two path transformation layers to significantly reduce path dimension while retaining the essential information fed into the Logsig-RNN module. (The network architecture is illustrated in Figure \ref{LP_Logsig_RNN} (Right).) Finally, numerical results demonstrate that replacing the RNN module by the Logsig-RNN module in SOTA networks consistently improves the performance on both Chalearn gesture data and NTU RGB+D 120 action data in terms of accuracy and robustness. In particular, we achieve the state-of-the-art accuracy on Chalearn2013 gesture data by combining simple path transformation layers with the Logsig-RNN. Codes are available at \url{https://github.com/steveliao93/GCN_LogsigRNN}.
	
\begin{figure}[!ht]
  \centering
  \begin{minipage}{0.45\textwidth}
		\includegraphics[width= 0.99 \textwidth]{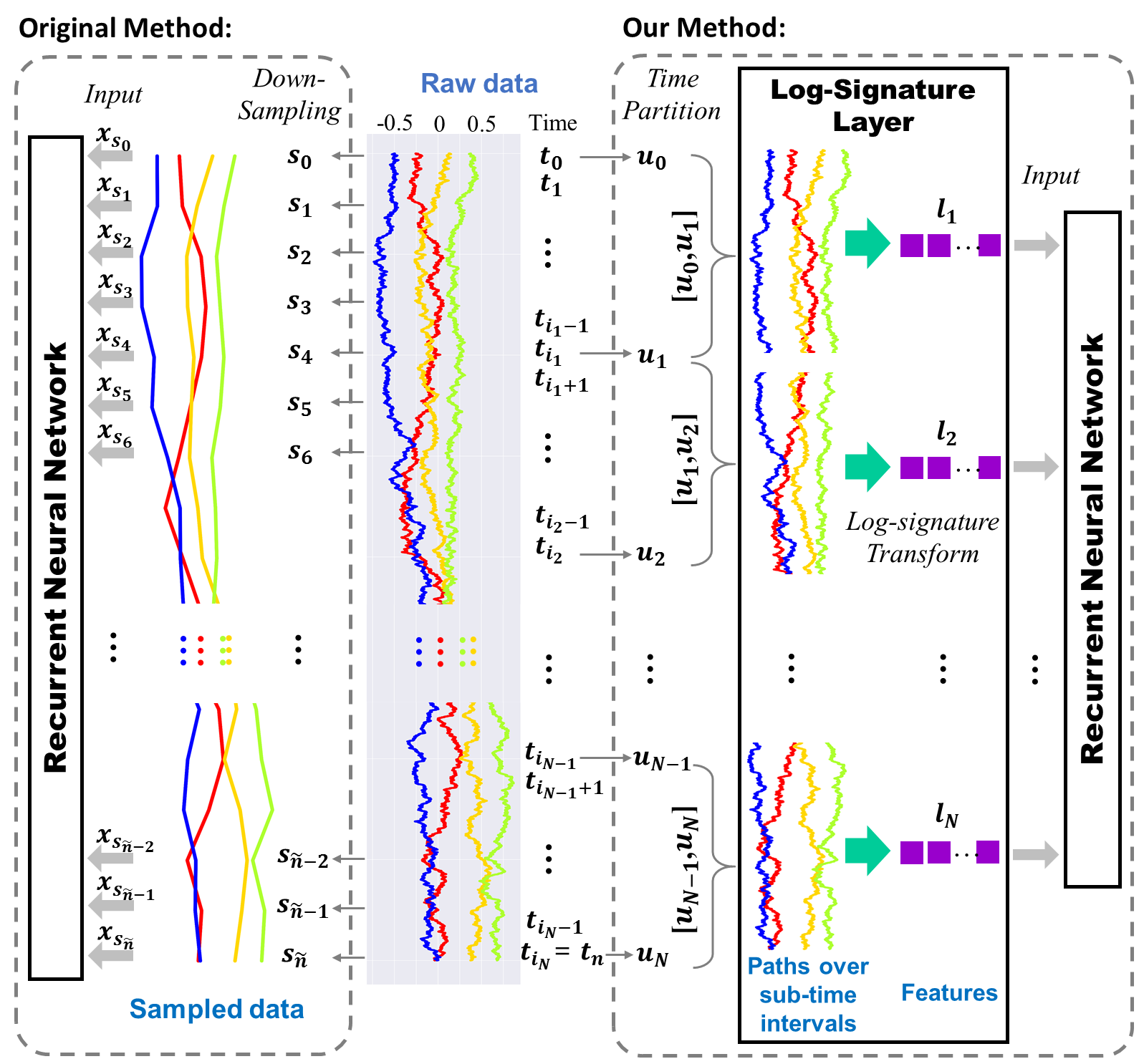}
\end{minipage}
\quad
  \begin{minipage}{0.5\textwidth}
  \includegraphics[width= 0.99\textwidth]{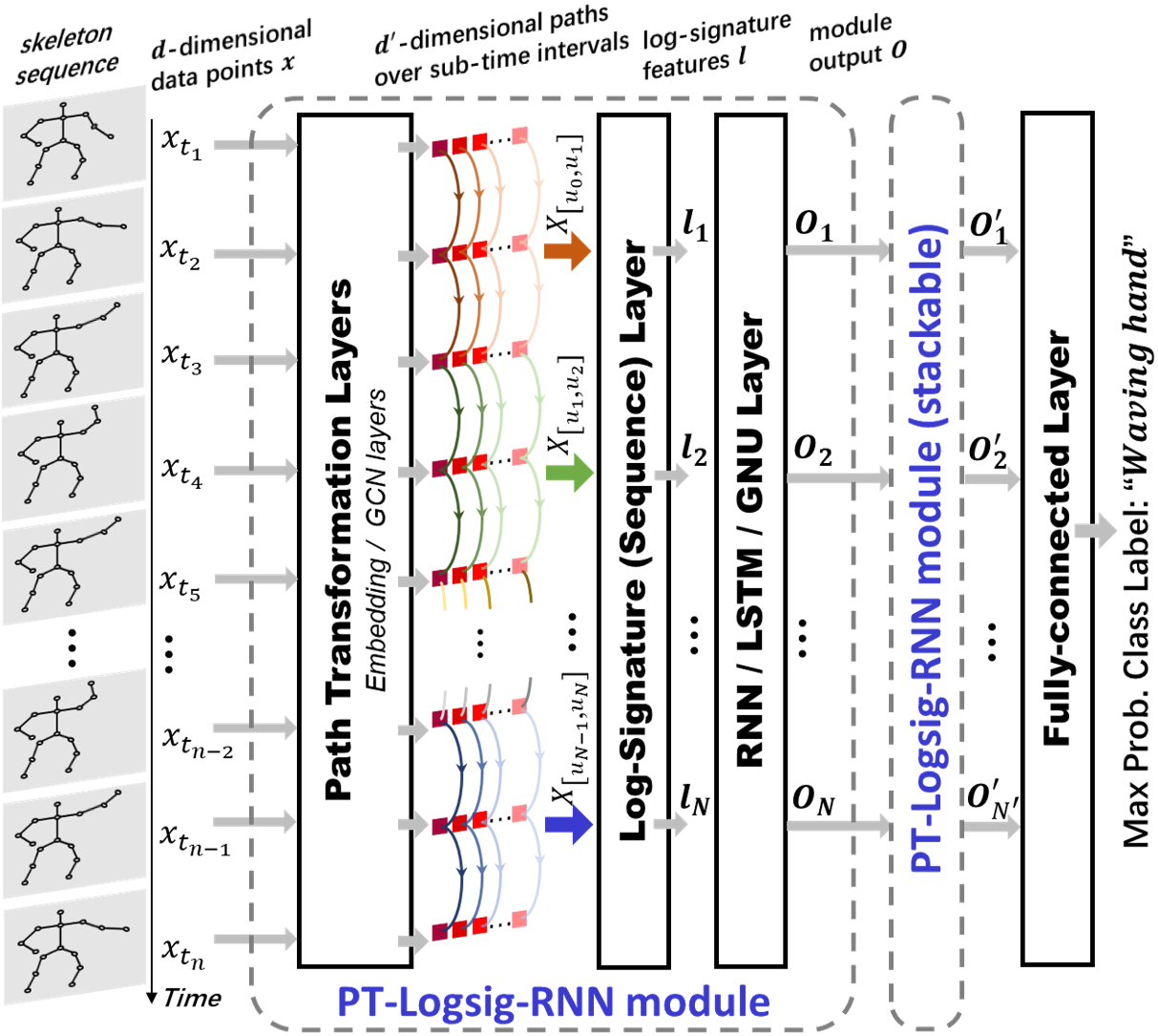}
\end{minipage}
\caption{(Left) Comparison of Logsig-RNN and RNN; (Right) Pipeline of the PT-Logsig-RNN module for skeleton-based human action recognition. This stackable module consists of Path Transformation Layers, followed by the Log-Signature Layer and an RNN-type layer. %The PT-Logsig-RNN module can be stacked. 
}\label{LP_Logsig_RNN}
%\caption{Comparison of Logsig-RNN and RNN.}\label{RNN_SDE}
\end{figure}
\end{abstract}

% keywords can be removed

\section{Introduction}
\label{sec:intro}
%\subsection{Motivation}
Human action recognition (HAR) in videos is a classical and challenging problem in computer vision with a wide range of applications in human-computer interfaces and communications. Low-cost motion sensing devices, e.g. Microsoft Kinect, and reliable pose estimation methods, are both leading to an increase in popularity of research and development on skeleton-based HAR (SHAR). Compared with RGB-D HAR, skeleton-based methods are robust to illumination changes and have benefits of data privacy and security. 

Although vast literature is devoted to SHAR \cite{lo20163d,wang2019comparative,ren2020survey}, the challenge remains open due to two main issues: (1) how to extract discriminative representations for the high dimensional spatial structure of skeletons; (2) how to model the temporal dynamics of motion. 

With the increasing development and impressive performance of deep learning models e.g. Recurrent Neural Networks (RNN) \cite{lev2016rnn,10.1007/978-3-319-46487-9_50,Liu2018SkeletonBasedHA,nturgb,Wang2017ModelingTD}, Convolutional Neural Networks (CNN) \cite{caetano2019skelemotion,cheron2015p,8306456,li2018co,liufsnet,liu2018recognizing}, and Graph Convolutional Networks (GCN) \cite{li2019actional,shi2019skeleton}, data-driven deep features have gained increasing attention in SHAR \cite{ren2020survey}. However, these methods are often data greedy and computationally expensive, and not well adapted to data of different sizes/lengths. For example, when the lengths of data sequences are long and diverse, long-short term memory networks (LSTMs) either suffer from tremendous training cost with heuristic padding or are forced to down-sample/re-sample the data, which potentially misses the microscopic information. 

To address some of the difficulties and better capture the temporal dynamics, we propose a simple but effective neural network module, namely Logsig-RNN, by blending the Log-signature (Sequence) Layer with the RNN layer, as shown in Figure \ref{LP_Logsig_RNN} (Left). The \emph{log-signature}, which was originally introduced in rough path theory in the field of stochastic analysis, is an effective mathematical tool to summarize and vectorize complex un-parameterized streams of multi-modal data over a \emph{coarse} time scale with a \emph{low dimensional} representation, reducing the number of timesteps in the RNN. The properties of the log-signature also allow to handle time series with variable length without the use of padding and provide robustness to missing data. This allows the following RNN layer to learn more expressive deep features, leading to a systematic method to treat the complex time series data in SHAR.

The spatial structure in SHAR methods is commonly modelled using coordinates of joints \cite{fan2016action,jhuang2013towards,yang2014effective}, using body parts to model the articulated system \cite{evangelidis2014skeletal,li2016multiview,shahroudy2015multimodal} or by hybrid methods using information from both joints and body parts \cite{ke2017skeletonnet,li2018co}. Inspired by~\cite{li2018co} and \cite{msg3d_liu}, we investigate combining the flexible Logsig-RNN with Path Transformation Layers (PT) which include an Embedding Layer (EL) to reduce the spatial dimension of pure joint information and a vanilla Graph Convolutional Layer (GCN) to learn to implicitly capture the discriminative joints and body parts.
%Inspired by~\cite{li2018co} which used different shapes of convolutional kernels, we use a stack of linear embedding layers to learn to implicitly capture the spatial discriminative joints and parts, and significantly reduce spatial dimension.  

Our pipeline for SHAR is illustrated in Figure \ref{LP_Logsig_RNN} (Right). With quantitative analysis on Chalearn2013 gesture dataset and NTU RGB+D
120 action dataset, we validated the efficiency and robustness of Logsig-RNN and
the effects of the Path Transformation layers.

%-------------------------------------------------------------------------
\section{Related Works on Signature Feature}
The signature feature (SF) of a path, originated from rough path theory \cite{lyons1998differential}, was introduced as universal feature for time-series modelling \cite{levin2013learning} and has been successfully applied to machine learning (ML) tasks, e.g. financial data analysis \cite{gyurko2013extracting,lyons2014feature}, handwriting recognition \cite{diehl2013rotation,graham2013sparse,yang2016dropsample,xie2018learning}, writer identification \cite{yang2016deepwriterid}, signature verification \cite{lai2017online}, psychiatric analysis \cite{arribas2018signature}, speech emotion recognition \cite{wang2019path} as well as action classification \cite{ahmad2019human,kiraly2016kernels,li2019skeleton,yang2017leveraging}. These SF-based methods can be grouped into the whole-interval manner and sliding-window manner. The whole-interval manner regards data streams of various lengths as paths over the entire time interval; then the SFs of fixed dimensions are computed to encode both global and local temporal dependencies  \cite{levin2013learning,diehl2013rotation,lai2017online,arribas2018signature,wang2019path,yang2017leveraging,li2019skeleton}. The sliding-window manner computes the SFs over window-based sub-intervals which are viewed as local descriptors and are further aggregated by deep networks \cite{graham2013sparse,yang2016dropsample,yang2016deepwriterid,xie2018learning,li2017lpsnet}. Our method falls into the second category using disjoint sliding windows. There are few works on using the log-signature, rather than the signature, in ML applications \cite{li2017lpsnet, morrill2021ICML}. In this paper, we demonstrate the properties, efficiency, and robustness of the log-signature compared with the signature. Recent work \cite{kidger2019deep} proposed to use the signature transformation as a layer rather than as a feature extractor. We propose a Log-signature (Sequence) Layer with impressive advantages in temporal modelling to improve RNNs. To our best knowledge, it is the first of the kind to (1) integrate the log-signature sequence with RNNs (2) as a differentiable layer which can be used anywhere within a larger model, instead of using the log signature as feature extractor. In particular, the output of the LogsigRNN is of the same shape as its input with a reduced time dimension.

\section{The Log-Signature of a Path}\label{SectionPreliminaries}
Let $E:=\mathbb{R}^{d}$, $J=[S, T]$ and $X: J \rightarrow E$ be a continuous path endowed with a norm denoted by $\vert \cdot \vert$. In practice we may only observe $X$ built at some fine scale out of time stamped
values $X^{\hat{\mathcal{D}}} = [X_{t_{1}}, X_{t_{2}}, \cdots, X_{t_{n}}]$, where $\hat{\mathcal{D}} = (t_{1}, \cdots, t_{n})$. Throughout this paper, we embed the discrete time series $X^{\hat{\mathcal{D}}}$ to a continuous path of bounded variation by linear interpolation for a unified treatment (See detailed discussion in Section 4 of \cite{levin2013learning}). Therefore, we focus on paths of bounded variation. In this section, we introduce the definition of the signature/log-signature. Then we summarize the key properties of the log-signature, which make it an effective, compact and high order feature of streamed data over time intervals. Lastly, we highlight the comparison between the log-signature and the signature. Further discussions and demo codes on the (log)-signature can be found in the supplementary material. 
%this \href{https://github.com/logsigRNN/learn_sde/blob/master/Pen-digit_learning/pendigit_demo.ipynb}{GitHub Demo}.%\footnote{ \url{https://github.com/logsigRNN/learn_sde/blob/master/Pen-digit_learning/pendigit_demo.ipynb}.}}.
%{\color{red} ?submit the jupyter notebook as part of supplementary material}
\subsection{The (log)-signature of a path}
The background information and practical calculation of the signature as a faithful feature set for un‐parameterized paths can be found in \cite{kidger2019deep, levin2013learning,chevyrev2016primer}. We introduce the formal definition of the signature in this subsection.
\begin{definition}[Total variation]
The total variation of a continuous path $X: J\to E$ is defined on the interval $J$ to be 
%\begin{equation}
    $\vert\vert X\vert\vert_{J}= \sup_{\mathcal{D}\subset J}\sum_{j=0}^{r-1}\left\vert X_{t_{j+1}}-X_{t_j}\right\vert$,
%\end{equation}
where the supremum is taken over any time partition of $J$, i.e. $\mathcal{D} = (t_{1}, t_{2}, \cdots, t_{r})$.  \footnote{
A time partition of $J$ is an increasing sequence of real numbers $\mathcal{D} = (t_{i})_{i =0}^{r}$ such that $S = t_{0} < t_{1}< \cdots < t_{r} = T$.}
\end{definition}
Any continuous path $X: J \to E$ with finite total variation, i.e. $\vert\vert X\vert\vert_{J}<\infty$, is called a path of bounded variation. Let $BV(J,E)$ denote the range of any continuous path mapping from $J$ to $E$ of bounded variation. 
%\begin{definition}[Bounded variation]
%A continuous path $X: J \to E$ is said to be of %bounded variation on the interval $J$ if and only if %its total variation is finite, i.e.  $\vert\vert %X\vert\vert_{J}<\infty.$ 
%\end{definition}

Let $T((E))$ denote the tensor algebra space endowed with the tensor multiplication and componentwise addition, in which the signature and the log signature of a path take values.

\begin{definition}[The Signature of a Path]
Let $X \in BV(J, E)$. Define the $k^{th}$ level of the signature of the path $X_{J}$ as $\mathbf{X}^{k}_{J} =\int_{S}^{T} \dots \int_{S}^{u_{2}} dX_{u_{1}} \otimes \dots\otimes dX_{u_{n}}.$ The signature of $X$ is defined as $S(X_{J}) = (1, \mathbf{X}_{J}^1, \dots, \mathbf{X}_{J}^k, \dots)$.
%\begin{equation}
%\end{equation}
Let $S_{k}(X_{J})$ denote the truncated signature of $X$ of degree $k$, i.e. $S_{k}(X_{J}) = (1, \mathbf{X}_J^1, \dots, \mathbf{X}_J^k).$
\end{definition}

%\textbf{It is highlighted that $S(x^{\mathcal{D}})$ is NOT the collection of all the monomials of discrete time series!} The dimension of all monomials of $x^{\mathcal{D}}$ grows with $\vert \mathcal{D} \vert$, while the dimension of $S_{k}(x^{\mathcal{D}})$ is invariant to $\vert \mathcal{D} \vert$.

%\subsection{The log-signature of a path}
Then we proceed to define the logarithm map in $T((E))$ in terms of a tensor power series as a generalization of the scalar logarithm. 
\begin{definition}[Logarithm map]\label{eqn_log}
	Let $a = (a_{0}, a_{1}, \cdots) \in T((E))$ be such that $a_{0} = 1$ and $t = a - 1$. Then the logarithm map denoted by $\log$ is defined as follows:
	\begin{eqnarray}
	\log(a) = \log(1+t) = \sum_{n = 1}^{\infty} \frac{(-1)^{n-1}}{n} t^{\otimes n}, \forall a \in T((E)).
	\end{eqnarray}
\end{definition}

%\begin{lemma}
	%The logarithm map is bijective on the domain $\{ a \in T((E))\vert a_{0} = 1\}$.  %The inverse of the logarithm map is the exponential map.
%\end{lemma}
\begin{definition}[The Log Signature of a Path] For $X \in BV(J, E)$, the log signature of a path $X$ denoted by $lS(X_{J})$ is the logarithm of the signature of the path $X$. Let $lS_{k}(X_{J})$ denote the truncated log signature of a path $X$ of degree $k$.
\end{definition}
The first level of the log-signature of a path $X$ is the increment of the path $X_{T} - X_{S}$. The second level of the log-signature is the signed area enclosed by $X$ and the chord connecting the end and start of the path $X$. There are three open-source python packages esig \cite{esig}, iisignature \cite{ReizensteinIisignature2018} and signatory \cite{kidger2020signatory} to compute the log-signature.  
\subsection{Properties of the log-signature}
\textbf{Uniqueness}: By the uniqueness of the signature and bijection between the signature and log-signature, it is proved that the log-signature determines a path up to tree‐like equivalence \cite{UniquenessOfSignature}. The log-signature encodes the order information of a path in a graded structure. Note that adding a monotone dimension, like the time, to a path can avoid tree‐like sections. \\
\textbf{Invariance under time parameterization}: We say that a path $\tilde{X}: J \rightarrow E$ is the time re-parameterization of $X:J \rightarrow E$ if and only if there exists a non-decreasing surjection $\lambda: J \rightarrow J$ such that $\tilde{X}_{t} = X_{\lambda(t)}$, $\forall t \in J$. 
%\begin{lemma}\label{LogSigTimeInvariance}
Let $X \in BV(J, E)$  and a path $\tilde{X}: J \rightarrow E$ be a time re-parameterization of $X$. Then it is proved that the log-signatures of $X$ and $\tilde{X}$ are  equal\cite{lyons1998differential}. This is illustrated in figure \ref{TimeParameterizationInvariance}, where speed changes result in different time series representation but the same log-signature feature. This is beneficial as human motions are invariant under the change of video frame rates. The log-signature feature can remove the redundancy caused by the speed of traversing the path, which brings massive dimensionality reduction. 
\begin{figure}[!htbp]
	\centering
	\begin{minipage}[c]{0.22\textwidth}
		% \centering
		\includegraphics[width= \textwidth]{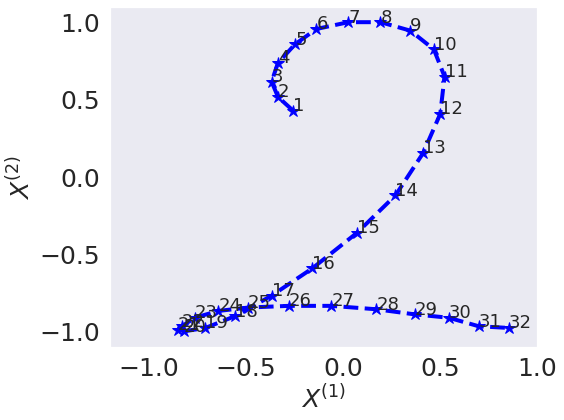}    
		%\caption{The sample of digit $8$}
		%\caption{Caption}
		%\label{fig:my_label}
	\end{minipage}
	\quad
	\begin{minipage}[c]{0.22\textwidth}
		% \centering
	%	\includegraphics[width= \textwidth]{no_param.png}
		\includegraphics[width= \textwidth]{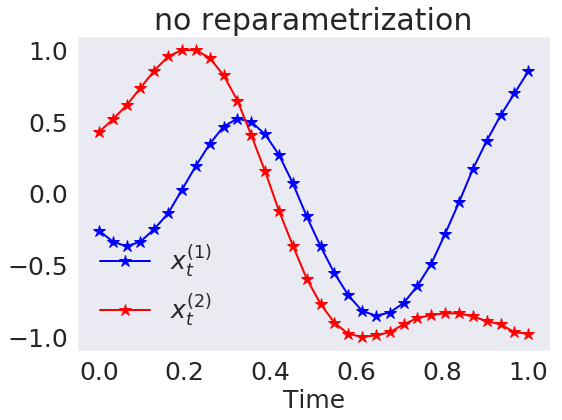}
		%\label{fig:my_label}
	\end{minipage}
	\quad
	\begin{minipage}[c]{0.22\textwidth}
		% \centering
	%	\includegraphics[width= \textwidth]{reparam1.png}
			\includegraphics[width= \textwidth]{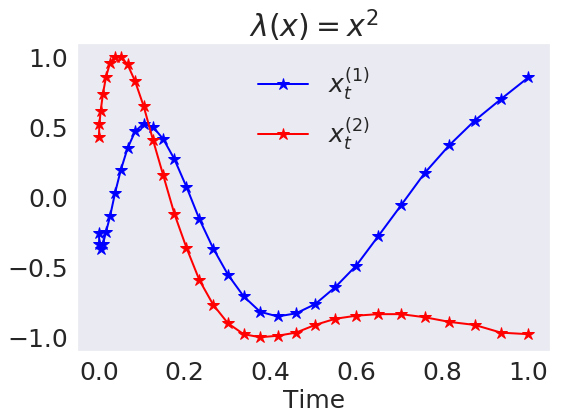}
		%\caption{Caption}
		%\label{fig:my_label}
	\end{minipage}
	\quad
	\begin{minipage}[c]{0.22\textwidth}
		% \centering
	%	\includegraphics[width= \textwidth]{reparam2.png}
			\includegraphics[width= \textwidth]{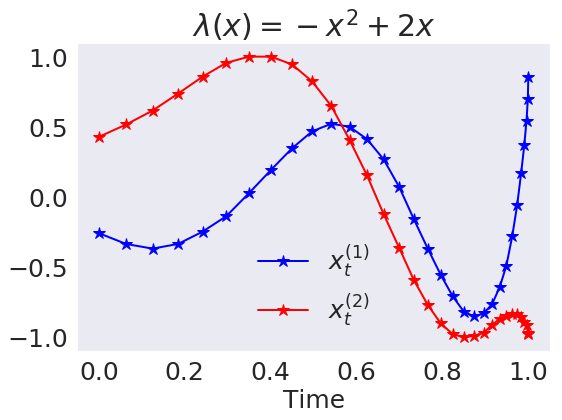}
		%\caption{Caption}
		%\label{fig:my_label}
	\end{minipage}
	\caption{The first figure represents the trajectory of the digit $2$, and the rest of figures plot the coordinates of the pen locations against time via different speed respectively, which share the same signature and log signature given in the first subplot. }\label{TimeParameterizationInvariance}
\end{figure}
$~$\\
\textbf{Irregular time series}: The truncated log-signature feature provides a robust descriptor of fixed dimension for time series of variable length, uneven time spacing and with missing data. For example, given a pen digit trajectory, random sub-sampling results in new trajectories of variable length and non-uniform spacing. In this case, the mean absolute percentage error (MAPE) of the log-signature is small (see Figure 3 in supplementary material). 

\subsection{Comparison between signature and log-signature }
The logarithm map is bijective on the domain $\{ a \in T((E))\vert a_{0} = 1\}$. Thus the log-signature and the signature is one-to-one. Therefore, the signature and log-signature share all the properties covered in the previous subsection. In the following, we highlight important differences between the signature and the log-signature.

\begin{wrapfigure}{r}{0.4\textwidth}
  %\begin{center}
    \includegraphics[width=0.38\textwidth]{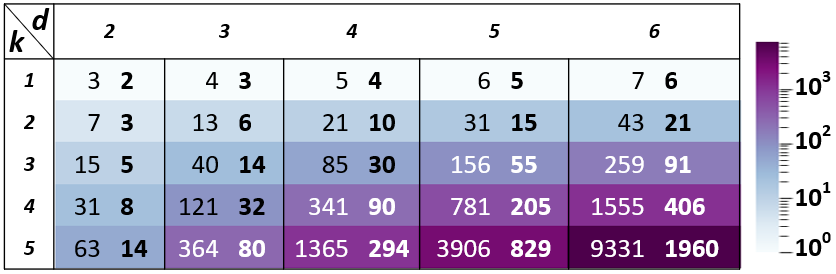}
  %\end{center}
  \caption{The dimension comparison between the signature and log-signature (in bold) of a $d$-dimensional path of degree $k$.} \label{Tab_sig_logsig_dim_comparison}
\end{wrapfigure}
The log-signature is a parsimonious representation for the signature feature, whose dimension is lower than that of the signature in general. For $d>2$, the dimension of the signature of a $d$-dimension path up to degree $k$ is $\frac{d^{k+1}-1}{d-1}$, and the dimension of the corresponding log-signature is equal to the necklace polynomial on $(d, k)$\cite{reizenstein2018iisignature}. Figure \ref{Tab_sig_logsig_dim_comparison} shows that the larger $d$ and $k$, the greater dimension reduction the log-signature brings over the signature (the colour represents the dimension gap between signature and the log signature). In contrast to the signature, the log-signature does not have universality, and thus it needs to be combined with non-linear models for learning.

\section{PT-Logsig-RNN Network}\label{section_logsig_RNN}
In this section, we propose a simple, compact and efficient PT-Logsig-RNN Network for SHAR, which is composed of (1) path transformation layers, (2) the Logsig-RNN module and (3) a fully connected layer. The overall PT-Logsig-RNN model is depicted in Figure \ref{LP_Logsig_RNN} (Right). We start by introducing the Log-Signature Layer and follow with the core module of our model, the Logsig-RNN module. In the end, we propose useful path transformation layers to further improve the performance of the Logsig-RNN module in SHAR tasks. 

\subsection{Log-Signature Layer}\label{BP_LogsignatureLayer}
We propose the Log-Signature (Sequence) Layer, which transforms an input data stream to a sequence of log-signatures over sub-time intervals. More specifically, consider a $d$-dimensional stream $x  \in BV(J, E)$ and let $\mathcal{D}:= (u_{k})_{k = 0}^{N}$ be a time partition of $J$. 

\begin{definition}[Log-Signature (Sequence) Layer]\label{def_logsig_seq}
A Log-Signature Layer of degree $M$ associated with $\mathcal{D}$ is a mapping from $BV(J, E)$ to $\mathbb{R}^{N \times d_{ls}}$ such that $\forall x \in BV(J,E)$, $x \mapsto (l_{k}^{M})_{k =0}^{N-1}$, where $l_{k}^{M}$ is the truncated log signature of $x_{[u_{k}, u_{k+1}]}$ of degree $M$, i.e. $l_{k}^{M}=lS_{M}(x_{[u_{k}, u_{k+1}]})$. Here $d_{ls}$ is the dimension of the log-signature of a $d$-dimensional path of degree $M$.
\end{definition}
In practice, the input stream $x$ is usually only observed at a finite collection of time points $\hat{\mathcal{D}}$, which can be non-uniform, high frequency and sample-dependent. By interpolation, embedding $x^{\hat{\mathcal{D}}}$ to the path space allows the Log-Signature Layer to treat each sample stream over $\mathcal{D}$ in a unified way. The output dimension of the Log-Signature Layer is $(N, d_{ls})$, which does not depend on the time dimension of the input streams. A higher frequency of input data would not cause any dimension issue, but it makes the computation of $l_{k}$ more accurate. The Log-Signature Layer can shrink the time dimension of the input stream effectively, while preserving local temporal information by using the log-signatures.

It is noted that the Log-Signature Layer does not have any trainable parameters, but allows backpropagation\footnote{The derivation and implementation details of the backpropagation through the Log-Signature Layer can be found in the supplementary material.} through it. We extend the work on the backpropagation algorithm of single log-signatures in \cite{reizenstein2018iisignature} to log-signature \emph{sequences}. Our implementation can accommodate time series samples of variable length over sub-time intervals, which may not be directly handled by the Log-Signature layer in the signatory package\cite{kidger2020signatory}.

\subsection{Logsig-RNN Network}\label{SectionInference}
Firstly, we introduce the conventional recurrent neural network. It is composed of three types of layers, i.e. the input layer $(x_{t})_t$, the hidden layer $(h_{t})_t$ and the output layer $(o_{t})_{t}$. A RNN takes an input sequence $x^{\hat{\mathcal{D}}} = (x_{t_{i}})_{i = 1}^{T}$ and computes an output $(o_{t})_{t = 1}^{T} \in \mathbb{R}^{T \times e}$ via
$h_t = \sigma(Ux_t + Wh_{t-1}), o_{t}= q(Vh_t)$, where $U$, $W$ and $V$ are model parameters, and $\sigma$ and $q$ are activation functions. Let $\mathcal{R}_\Theta((x_{t})_{t})$ denote the RNN model with $(x_{t})_{t}$ as the input and $\Theta:=\{U, W, V\}$ its parameter set. It is noted that this represents all the recurrent type neural networks, including LSTM, GRU, etc. Then we propose the following Logsig-RNN model.  

\begin{model}[Logsig-RNN Network]\label{logsigrnn}
Given $\mathcal{D}:= (u_{k})_{k = 0}^{N}$, a Logsig-RNN network computes a mapping from an input path $x \in BV(J, E)$ to an output defined as follows:
\begin{itemize}
    \item Compute $(l_{k})_{k =0}^{N-1}$ as the output of the Log-Signature Layer of degree $M$ associated with $\mathcal{D}$ for an input $x$.
 \item The output layer is computed by $\mathcal{R}_{\Theta}((l_{k})_{k = 0}^{N-1})$, where $\mathcal{R}_{\Theta}$ is a RNN type network.
\end{itemize}
\end{model}

%\subsection{Logsig-RNN algorithm}\label{SectionAlgorithm}
The Logsig-RNN model (depicted in Figure \ref{LP_Logsig_RNN} (Left)) is a natural generalization of conventional RNNs. When $\mathcal{D}$ coincides with timestamps of the input data, the Logsig-RNN Model with $M=1$ is the RNN model with the increments of the data as input. One main advantage of our method is to reduce the time dimension of the RNN model significantly as we use the principled, non-linear and compact log-signature features to summarize the data stream locally. It leads to higher accuracy and efficiency compared with the standard RNN model. Logsig-RNN can overcome the limitation of Sig-OLR \cite{levin2013learning} on stability and efficiency issues by using compact log-signature features and more effective non-linear RNN models. Compared with conventional RNNs the Logsig-RNN model has the same input and output structure.
\subsection{Path Transformation Layers}\label{Ch_LP_Logsig_RNN}
To more efficiently and effectively exploit the spatio-temporal structure of the path, we further investigate the use of two main path transformation layers (i.e. Embedding Layer and Graph Convolutional Layer) in conjunction with the Log-Signature Layer. 

A skeleton sequence $X$ can be represented as
a $n \times F \times D$ tensor (landmark sequence) and a $F \times F$ matrix $\mathcal{A}$ (bone information), where $n$ is the number of frames in the sequence, $F$ is the number of joints in the skeleton, $D$ is the coordinate dimension and $\mathcal{A}$ is the adjacency matrix to denote whether two joints have a bone connection or not. 

%For the skeleton-based action recognition, we propose the PT-Logsig-RNN model, which is composed of Path Transformation (PT) module followed by the Logsig-RNN. For landmark only data, we propose to use the EL-Logsig-RNN model, which is PT-Logsig-RNN with the Path Transformation Module constructed by subsequently adding EL, AL and TL. For the skeleton data with bone information, we propose to add the GCN to the Logsig-RNN as described before .

\textbf{Embedding Layer (EL)} In the literature, many models only use landmark data without explicit bone information. One can view a skeleton sequence as a single path of high dimension($d$) (e.g. a skeleton of 25 3D joints has $d= F \cdot D = 75$). Since the dimension of the truncated log-signature grows fast w.r.t.\ $d$, we add a linear Embedding Layer before the Log-Signature Layer to reduce the spatial dimension and avoid this issue. Motivated by \cite{li2018co}, we first apply a linear convolution with kernel dimension 1 along the time and joint
dimensions to learn a joint level representation. Then we apply full convolution on the second and third coordinates to learn the interaction between different joints for an implicit representation of skeleton data. The output tensor of EL has the shape $n \times d_{el}$, where $d_{el}$ is a hyper-parameter to control spatial dimension reduction. One can view the embedding layer as a learnable path transformation that can help to increase the expressivity of the (log)-signature.

In practice, the Embedding Layer is more effective when subsequently adding the Time-Incorporated Layer (TL) and the Accumulative Layer (AL). The details of TL and AL can be found in Section 5 in the appendix. For simplicity we will use EL to denote the Embedding Layer composed with TL and AL in the below numerical experiments. 

\textbf{Graph Convolutional Layer (GCN)}
Recently, graph-based neural networks have been introduced and achieved SOTA accuracy in several SHAR tasks due to their ability to extract spatial information by incorporating additional bone information using graphs. We demonstrate how a GCN and the Logsig-RNN can be combined to form the GCN-Logsig-RNN to model spatio-temporal information.

First, we define the GCN layer on the skeleton sequence. Let $G_{\theta}$ denote a graph convolutional operator $F \times D \rightarrow F \times \tilde{D}$ associated with $\mathcal{A}$ by mapping $x$ to $(\Gamma^{-\frac{1}{2}} (\mathcal{A} + I) \Gamma^{-\frac{1}{2}})x\theta$, where $\Gamma^{ii} = \sum_{j} (\mathcal{A}^{ij} + I^{ij})$, and $I$ is the identity matrix. Then we extend $G_{\theta}$ to the skeleton sequence by applying $G_{\theta}$ to each frame $X_{t}$, i.e. $G_{\theta}: X = (X_{t})_{t = 1}^{n} \mapsto (G_{\theta}(X_t))_{t = 1}^{n}$ to obtain an output as a sequence of graphs of time dimension $n$ with the adjacency matrix $\mathcal{A}$.    

Next we propose the below \emph{GCN-Logsig-RNN} to combine GCN with the Logsig-RNN. Let $\hat{X}^{(i)}_{t} \in \mathbb{R}^{\tilde{D}}$ denote the features of the $i^{th}$ joint of the GCN output $G_{\theta}(X_t)$ at time $t$. For each $i^{th}$ joint, $\hat{X}^{(i)} = (\hat{X}^{(i)}_{t})_{t = 1}^{n}$ is a $\tilde{D}$-dimensional path. We apply the Logsig-RNN to $\hat{X}^{(i)}$ as the feature sequence of each $i^{th}$ joint, and hence obtain a sequence of graphs whose feature dimension is equal to the log-signature dimension and whose time dimension is the number of segments in Logsig-RNN. This in particular also allows for the module to be stacked.

\section{Numerical Experiments}\label{section_numerics}
%and RNN$_{\mathcal{D}}$\footnote{RNN$_{\mathcal{D}}$ means the RNN model with $(s_{k})_{k = 1}^{\vert\mathcal{D}\vert}$ with $s_{k} = (x_{t_{i}})_{t_{i} \in [u_{k}, u_{k+1}]}$ as the input.}resp.
%\subsection{Datasets and Network Architecture}

We evaluate the proposed EL-Logsig-LSTM model on two datasets: (1) Charlearn 2013 data, and (2) NTU RGB+D 120 data. \textbf{Chalearn 2013 dataset} \cite{Escalera2013MultimodalGR} is a publicly available dataset for gesture recognition, which contains 11,116 clips of 20 Italian gestures performed by 27 subjects. Each body consists of 20 3D joints. %Here, we only use skeleton data (20 3D joints) for the gesture recognition. 
\textbf{NTU RGB+D 120} \cite{Liu_2019_NTURGBD120} is a large-scale benchmark dataset for 3D action recognition, which consists of $114,480$ RGB+D video samples that are captured from 106 distinct human subjects for 120 action classes. 3D coordinates of 50 joints in each frame are used in this paper. In our experiments, we validate the performance of our model using \emph{only} the skeleton data of the above datasets.\footnote{We implemented the Logsig-LSTM network and all the numerical experiments in both Tensorflow and Pytorch.} %It runs on a computer equipped with GeForce RTX 2080 Ti GPU.} the proposed PT-Logsig-RNN
\subsection{Chalearn2013 data}
\textbf{State-of-the-art performance}: We apply the EL-Logsig-LSTM model to Chalearn2013 and achieve state-of-the-art (SOTA) classification accuracy shown  in Table \ref{ChaLearn2013_STOA_Sens_analysis} of the 5-fold cross validation results. The EL-Logsig-LSTM ($M=2, N=4$) with data augmentation achieves performance comparable to the SOTA \cite{liao2019multi}.  %outperforms the previous SOTA methods \cite{li2019skeleton} by 0.86 \textit{pp}.
%{\fontsize{500}{600}(a) }

\begin{table}[!ht]
%	\centering
	\begin{minipage}{0.5\textwidth}
	\scalebox{0.75}{
		\begin{tabular}{|l|c|c|}
		%	\hline
	\multicolumn{3}{c}{\textbf{{\fontsize{12.5}{13}\selectfont (a) Accuracy comparison}}}\\	%\multicolumn{3}{c}{\textbf{Chalearn 2013 Data}}\\
			\hline
			\textbf{Methods}                     & \textbf{Accuracy(\%)}   & \textbf{Data Aug.}   \\ \hline
			Deep LSTM \cite{nturgb}               & 87.10   & $-$      \\
			Two-stream LSTM \cite{Wang2017ModelingTD}           & 91.70      & $\surd$     \\
			ST-LSTM + Trust Gate \cite{8101019}& 92.00 & $\surd$			\\
			3s\_net\_TTM \cite{li2019skeleton}              & 92.08  &  $\surd$  \\
			\textbf{Multi-path CNN}\cite{liao2019multi} & \textbf{93.13}& $\surd$
			\\ \hline
			LSTM$_0$                   & 90.92    & $\times$       \\
			LSTM$_0$ (+data aug.)                    &   91.18  & $\surd$       \\
			\hline
			EL-Logsig-LSTM & 91.77 $\pm$ 0.34 & $\times$ \\ 
			\textbf{EL-Logsig-LSTM(+data aug.)} & \textbf{92.94 $\pm$ 0.21} & $\surd$ \\ \hline
			GCN-Logsig-LSTM & 91.92 $\pm$ 0.28 & $\times$ \\
			GCN-Logsig-LSTM(+data aug.) & 92.86 $\pm$ 0.23 & $\surd$ \\ \hline
	\end{tabular}}
\end{minipage}
\quad
	 \begin{minipage}{0.35 \textwidth}
	 ~~~
	   \scalebox{0.7}{
    \begin{tabular}{|r|c|c|c|}
    		%	\hline
			\multicolumn{4}{c}{\textbf{{\fontsize{12.5}{13}\selectfont(b) Effects of EL}}}\\
    \hline
   Methods & $D_{el}$ & Accuracy(\%) & \# Trainable weights \\
    \hline
    \multirow{6}{*}{With EL}&10    & 91.09 & 120,594 \\
    
    &20    & 92.92 & 213,574 \\
    
    &30    & \textbf{93.38} & 357,954 \\
    
    &40    & 93.10 & 553,734 \\
    
    &50    & 93.33 & 800,914 \\
    
    &60    & 93.30 & 1,099,494 \\
    \hline
    W/O EL &- & 91.51 & 985,458\\
    \hline
    \end{tabular}}
%  \caption{The sensitivity study of the Logsig-RNN model against the number of segments on both the Chalearn 2013 and NTU RGB+D 120 datasets.}
\quad
%\centering
\text{                            }~~~  \scalebox{0.68}{
    \begin{tabular}{|c|c|c|c|}
 %   \hline
  %  \multicolumn{2}{|c|}{\multirow{2}[4]{*}{Accuracy(\%)}} & 
  \multicolumn{4}{c}{\textbf{{\fontsize{12.5}{13}\selectfont (c) Effects of number of Segments ($N$)}}} \\
%\cline{3-8}
   \hline
$N$ & 2     & \multicolumn{1}{c|}{4} & \multicolumn{1}{c|}{8}  \\
     %\cline{2-4}
     \hline
 Accuracy   & \multicolumn{1}{c|}{92.10$\pm$0.04} & \textbf{92.94$\pm$0.21} & 92.69$\pm$0.11   \\
\hline
$N$ &   \multicolumn{1}{c|}{16} & \multicolumn{1}{c|}{32} & \multicolumn{1}{c|}{64}\\
  \hline
  Accuracy    & 92.87$\pm$0.15 & 91.66$\pm$0.39 & 91.50$\pm$0.39\\
 \hline   
    \end{tabular}}%
    \end{minipage}

\caption{The accuracy comparison and sensitivity analysis on Chalearn2013. (a) The number after $\pm$ is the standard deviation of the accuracy. (b) $D_{el}$ is the spatial dimension of EL output.}\label{ChaLearn2013_STOA_Sens_analysis}
	 	
\end{table}

\noindent\textbf{Investigation of path transformation layers}: To validate the effects of EL, we compare the test accuracy and number of trainable weights in our network with and without EL on Chalearn 2013 data. Table \ref{ChaLearn2013_STOA_Sens_analysis} (b) shows that the addition of EL increases the accuracy by $1.87$ percentage points (\textit{pp}) while reducing the number of trainable weights by over $60\%$. Let $D_{el}$ denote the spatial dimension of the output of EL. We can see that even introducing EL without a reduction in dimensionality, i.e. setting $D_{el}$ to the original spatial dimension of 60, improves the test accuracy. Decreasing the dimensionality can lead to further improvements, with the best results in our experiments at $D_{el}=30$ with a test accuracy of $93.38\%$. A further decrease of $D_{el}$ leads to the performance deteriorating. The high accuracy of our model using EL to reduce the original spatial dimension from $60$ to $D_{el} = 30$ suggests that EL can learn implicit and effective spatial representations for the motion sequences. AL and TL contribute a 0.86 \textit{pp} gain in test accuracy to the EL-Logsig-LSTM model. 

\noindent\textbf{Investigation of different segment numbers in Logsig-LSTM}: Table~\ref{ChaLearn2013_STOA_Sens_analysis} (c) shows that increasing the number of segments ($N$) up to certain threshold increases the test accuracy, and increasing $N$ further worsens the model performance. For Chalearn2013, the optimal $N$ is $4$ and the optimal network architecture is depicted in Table A.1 in the supplement material.
% \ref{gesturearche} 
%\input{SupplementaryMaterial.tex}

\subsection{NTU RGB+D 120 data}
For NTU 120 data, we apply the EL-Logsig-LSTM, GCN-Logsig-LSTM and a stacked two-layer GCN-Logsig-LSTM(GCN-Logsig-LSTM$^2$) to demonstrate that the Logsig-RNN can be conveniently plugged into different neural networks and achieve competitive accuracy. 

Among non-GCN models, for X-Subject protocol, our EL-Logsig-LSTM model outperforms other methods, while it is competitive with \cite{caetano2019skelemotion} and \cite{liu2018recognizing} for X-Setup. The latter leverages the informative pose estimation maps as additional clues. Table \ref{tableofntu} (Left) shows the ablation study of EL-Logsig-LSTM For the X-Subject task, adding EL layer results in a 0.7 \textit{pp} gain over the baseline and the Logsig layer further gives a 5.9 \textit{pp} gain. 
\begin{table}[!ht]
		\begin{minipage}{0.45\textwidth}
		\scalebox{0.65}{
\begin{tabular}{|l|c|c|}
%	 \multicolumn{3}{|c|}{\textbf{NTU RGB+D 120 Data}}\\
\hline
\textbf{Methods}                     & \textbf{X-Subject(\%)}&
\textbf{X-Setup(\%)}\\
\hline
%Dynamic Skeleton\cite{7784788} & 50.8 & 54.7\\
ST LSTM\cite{10.1007/978-3-319-46487-9_50} & 55.7 & 57.9\\
FSNet\cite{liufsnet} & 59.9& 62.4\\
TS Attention LSTM\cite{Liu2018SkeletonBasedHA} & 61.2 & 63.3\\
%MT-CNN + RotClips\cite{8306456} & 62.2 & 61.8\\
Pose Evolution Map\cite{liu2018recognizing} & 64.6 & \textbf{66.9}\\
Skelemotion\cite{caetano2019skelemotion} & \textbf{67.7} & \textbf{66.9}\\
\hline
LSTM (baseline) & 60.9 $\pm$ 0.47& 57.6 $\pm$ 0.58\\
\hline

EL-LSTM & 61.6 $\pm$ 0.32& 60.0 $\pm$ 0.35\\
\textbf{EL-Logsig-LSTM} & \textbf{67.7 $\pm$ 0.38}& \textbf{66.9 $\pm$ 0.47}\\
\hline
\end{tabular}}
\end{minipage}
\quad
		\begin{minipage}{0.45\textwidth}
		\scalebox{0.65}{
\begin{tabular}{|l|c|c|}
%	 \multicolumn{3}{|c|}{\textbf{NTU RGB+D 120 Data}}\\
\hline
\textbf{Methods}                     & \textbf{X-Subject(\%)}&
\textbf{X-Setup(\%)}\\
\hline
RA-GCN\cite{ra_gcn_song}& 81.1 & 82.7\\
%2s-AGCN\cite{agcn_shi} & 82.9 &84.9 \\
4s Shift-GCN\cite{shiftgcn_Cheng}& 85.9 & 87.6\\
\textbf{MS-G3D Net}\cite{msg3d_liu} & 86.9 & \textbf{88.4}\\
\textbf{PA-Res-GCN}\cite{pa_res_song} & \textbf{87.3} & 88.3\\
\hline
(GCN-LSTM) &69.4 $\pm$ 0.46&71.4 $\pm$ 0.30\\
(GCN-LSTM)$^2$ & 72.1 $\pm$ 0.53 & 74.9 $\pm$ 0.27\\

\hline
GCN-Logsig-LSTM & 70.9 $\pm$ 0.22 & 72.4 $\pm$ 0.33\\
%*(GCN+Logsig+LSTM) + (GCN+LSTM)&&\\
\textbf{(GCN- Logsig-LSTM)$^2$}& \textbf{75.8 $\pm$ 0.35} & \textbf{78.0 $\pm$ 0.46}\\
\hline
\end{tabular}}
$~$
\end{minipage}
	\caption{Comparison of the accuracy ($\pm$ standard deviation) on NTU RGB+D120 Data.}
	\label{tableofntu}
\end{table}

When changing the EL to GCN in EL-Logsig-LSTM, we improved the accuracy by 3.2 \textit{pp} and 5.5 \textit{pp} for X-Subject and X-Setup tasks respectively. By stacking two layers of the GCN-Logsig-LSTM, we further improve the accuracy by 4.9 \textit{pp} and 5.6 \textit{pp}. The SOTA GCN models (\cite{msg3d_liu, pa_res_song}) have achieved superior accuracy, which is about 11 \textit{pp} higher than our best model. This may result from the use of multiple input streams (e.g. joint, bones and velocity) and more complex network architecture (e.g. attention modules and residual networks). Notice that our EL-Logsig-LSTM is flexible enough to allow incorporating other advanced techniques or combining multimodal clues to achieve further improvement. 
% (e.g. pose confidence score) (e.g. data augmentation and attention module)
\subsection{Robustness Analysis}
To test the robustness of each method in handling missing data and varying frame rate, we construct new test data by randomly discarding/repeating a certain percentage ($r$) of frames from each test sample, and evaluate the trained models on the new test data. Figure \ref{gesturemissingdata} (Left) shows that the proposed EL-Logsig-LSTM exhibit only very small drops in accuracy on Chalearn2013 as $r$ increases while the accuracy of the baseline drops significantly. We start to see a more significant drop in accuracy in our models only as we reach a drop rate of $50\%$. Figure \ref{gesturemissingdata} (Right) shows that the same is true for the proposed GCN-Logsig-LSTM model on the NTU data. Compared with GCN-LSTM (baseline) and the SOTA model MSG3D Net \cite{msg3d_liu} it is clearly more robust, at a drop rate of 50\% or more it even outperforms MSG3D Net which has a 10 $\textit{pp}$ higher accuracy than our model at $r=0$. This demonstrates that both EL-Logsig-LSTM and GCN-Logsig-LSTM are significantly more robust to missing data than previous models.

\begin{figure}[!ht]
	\centering
	 \begin{minipage}{0.4 \textwidth}
	 \includegraphics[width = 1.1\textwidth]{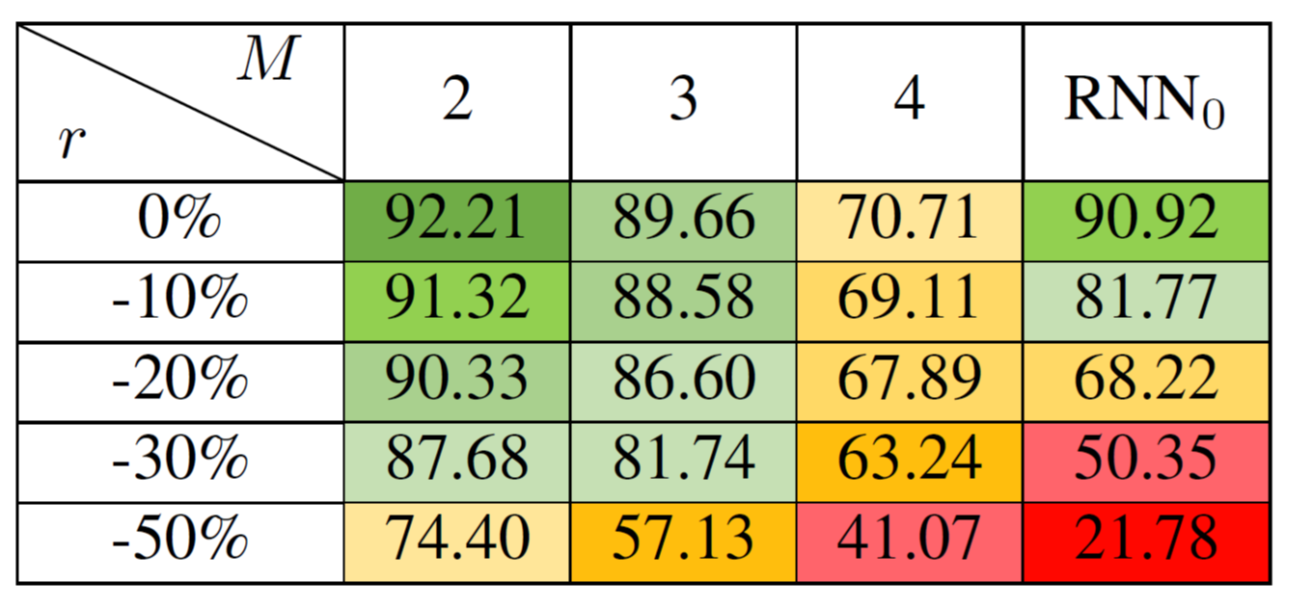}

	 \end{minipage}
	 \quad
	 \begin{minipage}{0.45 \textwidth}
%	  \begin{figure}[!ht]
	       \centering
	       \includegraphics[width = 0.95\textwidth]{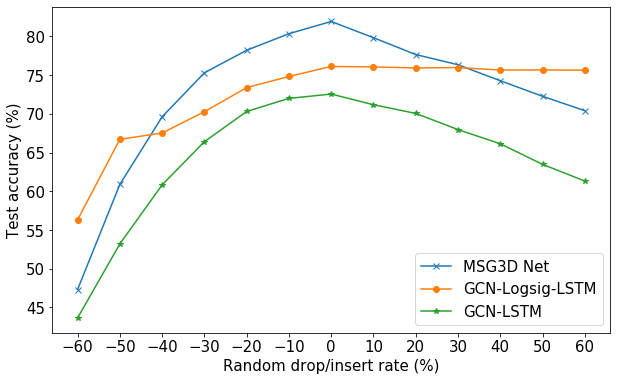}

	       \label{fig:my_label}
%	   \end{figure}
	 \end{minipage}
	 	\caption{The accuracy (\%) on the new test sets with various drop/insert rates ($r$). (Left) Chalearn2013. $N=4$, and no data augmentation is used. (Right) NTU RGB+D 120 data.}
	\label{gesturemissingdata}%
\end{figure}%

\subsection{Efficiency Analysis}
To demonstrate that the log-signature can help reduce the computational cost of backpropagating through many timesteps associated with RNN-type models we compare the training time and accuracy of a standard single LSTM block with a Logsig-LSTM using the same LSTM component on the ChaLearn dataset. To evaluate the efficiency as the length of the input sequence grows we linearly interpolate between frames to generate longer input sequences. We can see in the results in Figure \ref{fig:efficiency} that, as the length of the input sequence grows, the time to train the Logsig-LSTM grows much slower than that of the standard LSTM. Moreover, the Logsig-LSTM retains its accuracy while the accuracy of the LSTM drops significantly as the input length increases. This shows that the addition of the log-signature helps with capturing long-range dependencies in the data by efficiently summarizing local time intervals and thus reducing the number of timesteps in the LSTM.

We also compare the performance of the log signature and the discrete cosine transformation (DCT), which was used in \cite{mao2019learning} for reduction of the temporal dimension. Both transformations can be computed as a pre-processing step. As can be seen in Figure \ref{fig:efficiency} in this case the log-signature leads to slightly longer training time than DCT due to a larger spatial dimension, but achieves a considerably higher accuracy. If the transformation is computed at training time the cost of DCT is comparable to the log-signature.
\begin{figure}[!htbp]
	\centering
	%\begin{minipage}[c]{0.8\textwidth}
		% \centering
		\includegraphics[width=0.9\textwidth]{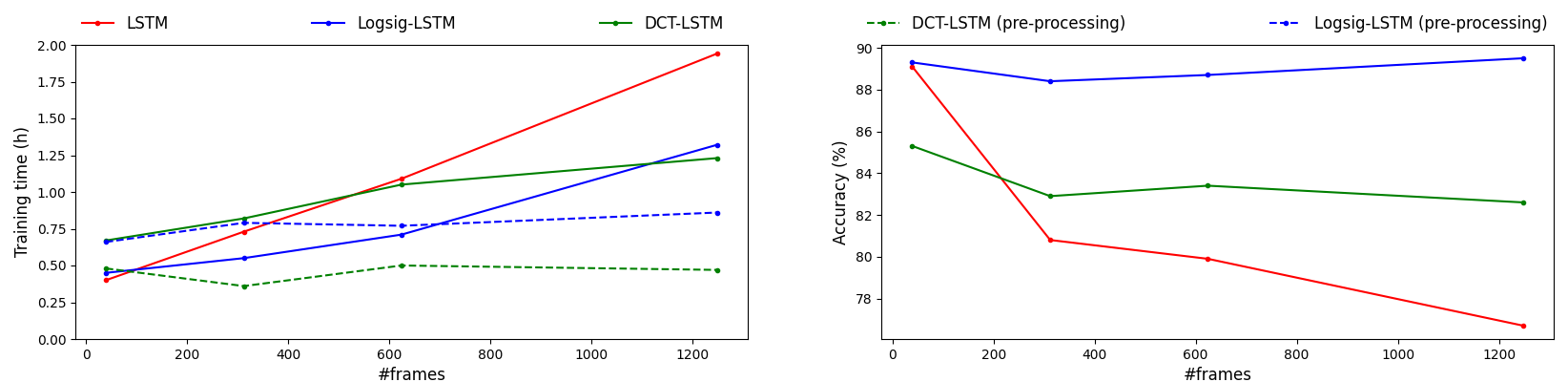} 
	%	\end{minipage}
%\quad
%	\begin{minipage}[c]{0.2\textwidth}
\caption{Comparison of training time and accuracy of standard LSTM and Logsig-LSTM.}\label{fig:efficiency}
%	\end{minipage}

\end{figure}
\vspace{-7mm}
\section{Conclusion}\label{SectionFutureWork}
We propose an efficient and compact end-to-end EL-Logsig-RNN network for SHAR tasks, providing a consistent performance boost of the SOTA models by replacing the RNN with the Logsig-RNN. As an enhancement of the RNN layer, the proposed Logsig-RNN module can reduce the time dimension, handle irregular time series and improve the robustness against missing data and varying frame rates. In particular, EL-Logsig-RNN achieves SOTA accuracy on Chalearn2013 for gesture recognition. For large-scale action data, the GCN-Logsig-RNN based models significantly improve the performance of EL-Logsig-RNN. Our model shows better robustness in handling varying frame rates. It merits further research to improve the combination with GCN-based models to further improve the accuracy while maintaining robustness. 
\section*{Acknowledgement}
All authors are supported by the Alan Turing Institute under the EPSRC grant EP/N510129/1.  T.L., W.Y., K. S. and H.N. are supported by the EPSRC under the program grant EP/S026347/1. W.Y. was supported by Royal Society Newton International Fellowship NIF/R1/18466. 

\section*{Appendix}
\appendix
The (log)-signature of a path is a core mathematical object in rough path theory, which is a branch of stochastic analysis. In this appendix, we give a brief overview of the signature and log signature of a path, and their properties which are useful in the context of machine learning. Besides we provide illustrative examples via the python notebook in the supplementary material. %github demo\footnote{{\color{red}?} \url{https://github.com/logsigRNN/learn_sde/blob/master/Pen-digit_learning/pendigit_demo.ipynb}}. 
In the following, for concreteness, we focus on paths of bounded variation, but the definition of the signature can be generalized to paths of finite $p$-variation. Interested readers can refer to \cite{lyons2007differential} for a rigorous introduction of the (log)-signature of a path.
\section{The signature of a path} Let us recall the definition of the signature of a path. Let $J:=[S, T]$ denote a compact interval and $E:=\mathbb{R}^{d}$. $BV(J, E)$ denotes the space of all continuous paths of finite length from $J$ to $E$. 
\begin{definition}[The Signature of a Path]
Let $X \in BV(J, E)$. Let $I=(i_{1}, i_{2}, \cdots, i_{n})$ be a multi-index of length $n$ where $i_{j} \in \{1, \cdots, d\}, \forall j \in \{1, \cdots, n\}$. Define the coordinate signature of the path $X_{J}$ associated with the index $I$ as follows:
\begin{equation*}
    X^{I}_{J} = \underset{\underset{u_{1}, \dots, u_{k} \in J}{u_{1} < \dots < u_{k}}} { \int \dots \int} dX_{u_{1}}^{(i_{1})}  \dots  dX_{u_{n}}^{(i_{2})}.
\end{equation*}
The signature of $X$ is defined as follows:
\begin{equation}
    S(X)_J = (1, \mathbf{X}_J^1, \dots, \mathbf{X}_J^k, \dots)
\end{equation}
where $\displaystyle \mathbf{X}_J^k = \underset{\underset{u_{1}, \dots, u_{k} \in J}{u_{1} < \dots < u_{k}}} { \int \dots \int} dX_{u_{1}} \otimes \dots \otimes dX_{u_{k}} = (X^{I}_{J})_{I = (i_{1}, \cdots, i_{k})}, \forall k \geq 1$. 

Let $S_{k}(X)_{J}$ denote the truncated signature of $X$ of degree $k$, i.e.
\begin{equation}
    S_{k}(X)_J = (1, \mathbf{X}_J^1, \dots, \mathbf{X}_J^k).
\end{equation}
\end{definition}

The signature of a path has a geometric interpretation. The first level signature $\mathbf{X}_J^1$ is the increment of the path $X$, i.e. $X_{T} - X_{S}$, while the second level signature represents the signed area enclosed by the curve $X$ and the cord connecting the start and end points of the path $X$.
 
The signature of $X$ arises naturally as the basis function to represent the solution to linear controlled differential equation based on the Picard's iteration \cite{lyons1998differential}. It plays the role of non-commutative monomials on the path space. In particular, if $X$ is a one-dimensional path, the $k^{th}$ level of the signature of $X$ can be computed explicitly by induction for every $k \in \mathbb{N}$ as follows
\begin{eqnarray}
 \mathbf{X}_J^k = \frac{(X_{T} - X_{S})^{k}}{k!}.
\end{eqnarray}
The signature of a $d$-dimensional linear path is given explicitly in the below lemma.
\begin{lemma}
Let $X: [S, T] \rightarrow E$ be a linear path. Then 
\begin{equation}
    S^{n}(X) = \frac{(X_{T} - X_{S})^{\otimes n}}{n!}.
\end{equation}
Equivalently speaking, for any multi-index $I = (i_{1}, \cdots, i_{n})$,
\begin{equation}
    S^{I} = \frac{\prod_{j=1}^{n}(X_{T}^{(i_{j})})}{n!}.
\end{equation}
\end{lemma}

The signature of a path can be interpreted as a solution to the controlled differential equation driven by a path in the tensor algebra space. 
\begin{theorem}\cite{lyons2007differential}\label{SigConrtolledEquation}
Let $X \in BV( [0, T], E)$. Define $f: T^{(n)} \rightarrow L(E, T^{n}(E))$ by
\begin{eqnarray}
f(a_{0}, a_{1}, \cdots, a_{n}) x = (0, a_{0} \otimes x, a_{1} \otimes x, \cdots, a_{n-1} \otimes x).
\end{eqnarray}
Then the unique solution to the differential equation
\begin{equation}
    dS_{t} = f(S_{t})dX_{t}, S_{0} = (1, 0, \cdots, 0),
\end{equation}
is the path $S: [0, T] \rightarrow T^{(n)}(E)$ defined for all $t \in [0, T]$ by
\begin{eqnarray*}
S_{t} = S_{n}(X_{[0, t]}) = (1,  \mathbf{X}_{[0, t]}^1, \dots, \mathbf{X}_{[0, t]}^n)).
\end{eqnarray*}
\end{theorem}
Formally the signature of the path $X_{0, t}$ as a function from $[0, T]$ to $T((E))$ is the solution to the following equation
\begin{eqnarray*}
d S(X)_{[0, t]} = S(X_{[0, t]}) \otimes dX_{t}, S(X_{[0, 0]}) = 1. 
\end{eqnarray*}
\subsection{Multiplicative Property}
The signature of paths of finite length (bounded variation) has the multiplicative property, also called Chen's identity. 
\begin{definition}
Let $X \in BV([0, s], E)$ and $Y \in BV([s, t], E)$ be two continuous paths. Their concatenation is the path denoted by $X * Y \in BV([0, t], E)$ defined by
\begin{eqnarray*}
(X*Y)_{u} = \begin{cases} X_{u}, &  u \in [0,s], \\ Y_{u}-Y_{s}+X_{s}, & u \in [s, t]. \end{cases}
\end{eqnarray*}
\end{definition}
\begin{theorem}[Chen's identity]\label{Chen}. Let  $X \in BV([0, s], E)$ and $Y \in BV([s, t], E)$. Then
\begin{eqnarray*}
S(X*Y) = S(X) \otimes S(Y).
\end{eqnarray*}
\end{theorem}
Chen's identity asserts that the signature is a homomorphism between the path space and the signature space.

The multiplicative properties of the signature allows us to compute the truncated signature of a piecewise linear path.
 
\begin{lemma}
Let $X$ be a $E$-valued piecewise linear path, i.e. $X$ is the concatenation of a finite number of linear  paths,  and  in  other  words  there exists a positive integer $l$ and linear paths $X_1,X_2, \cdots,X_l$ such that $X=X_1*X_2* \cdots *X_l$. Then
\begin{eqnarray}
S(X) = \otimes_{i = 1}^{l} \exp(X_{i}).
\end{eqnarray}
\end{lemma}

\subsection{Uniqueness of the signature}
Let us start with introducing the definition of the tree-like path.
\begin{definition}[Tree-like Path]\label{def_tree_like}
A path $X \in BV(J, E)$ is tree-like if there exists a continuous function $h: J \rightarrow [0, +\infty)$ such that $h(S) = h(T) = 0$ and such that, for all $s, t \in J$ with $s \leq t$, 
\begin{eqnarray*}
\vert\vert X_{t} - X_{s}\vert\vert \leq h(s) + h(t) - 2 \inf_{u \in [s, t]} h(u).
\end{eqnarray*}
\end{definition}
Intuitively a tree-like path is a trajectory in which there is a section where the path exactly retraces itself. The tree-like equivalence is defined as follows: we say that two paths $X$ and $Y$ are the same up to the tree-like equivalence if and only if the concatenation of $X$ and the inverse of $Y$ is tree-like. Now we are ready to characterize the kernel of the signature transformation.
\begin{theorem}[Uniqueness of the signature]\label{UniquenessOfSig}
	Let $X \in BV(J,E)$ . Then $S(X)$ determines $X$ up to the tree-like equivalence defined in Definition \ref{def_tree_like}.\cite{UniquenessOfSignature}  %\cite{boedihardjo2014signature},   
\end{theorem}
Theorem \ref{UniquenessOfSig} shows that the signature of the path can recover the path trajectory under a mild condition. The uniqueness of the signature is important, as it ensures it to be a discriminative feature set of un-parameterized streamed data.
\begin{remark}
	A simple sufficient condition for the uniqueness of the signature of a path of finite length is that one component of $X$ is monotone. Thus the signature of the time-joint path determines its trajectory (see \cite{levin2013learning}). 
\end{remark}

\subsection{Invariance under time parameterization}
\begin{lemma}[Invariance under time parameterization]\label{SigTimeInvariance}\cite{lyons2007differential}
	Let $X \in V_{1}(J, E)$ and a path $\tilde{X}: J \rightarrow E$ be a time re-parameterization of $X$. Then 
	\begin{eqnarray}\label{sig_time_inv}
	S(X_{J}) = S(\tilde{X}_{J}).
	\end{eqnarray} 
\end{lemma}
Re-parameterizing a path inside the interval does not change its signature. As can be seen in Figure \ref{TimeParameterizationInvariance}, speed changes result in different time series representation but the same signature feature. It means that signature feature can reduce dimension massively by removing the redundancy caused by the speed of traversing the path. It is very useful for applications where the output is invariant w.r.t. the speed of an input path, e.g. online handwritten character recognition and video classification. 

\subsection{Shuffle Product Property} We introduce a special class of linear forms on $T((E))$; Suppose $(e^{*}_{n_{1}}, \cdots, e^{*}_{n_{d}}, \cdots) $ are elements of $E^{*}$. We can introduce coordinate iterated integrals by setting
$X^{(i)}_{u}:= \langle e^{*}_{i}, X_{u} \rangle $,
and rewriting 
$\langle e_{i_{1}}^{*} \otimes \cdots \otimes e_{i_{n}}^{*}, S(X) \rangle$ as the scalar iterated integral of coordinate projection. In this way, we realize the $n^{th}$ degree coordinate iterated integrals as the restrictions of linear functionals in $E^{\otimes n}$ to the space of signatures of paths. If $(e_1,\cdots,e_d )$ is a basis for a finite dimensional space $E$, and $(e_1^{*}, \cdots ,e_d^{*})$ is a basis for the dual $E^{*}$ it therefore follows that
\begin{eqnarray*}
\mathbf{X}_{J} = \sum_{ \substack{k \geq 0\\i_{1}, \cdots, i_{k}  \\ \in 1, 2, \cdots, d\}}}\underset{\underset{u_{1}, \dots, u_{k} \in J}{u_{1} < \dots < u_{k}}} { \int \dots \int} dX_{u_{1}}^{(i_{1})} \otimes \dots \otimes dX_{u_{k}}^{(i_{k})} e_{n_{1}}\otimes \cdots \otimes e_{n_{k}}.
\end{eqnarray*}

\begin{theorem}[Shuffle Algebra]
The linear forms on $T((E))$ induced by $T(E^{*})$, when restricted to the range $S(BV([0, T], E)$ of the signature, form an algebra of real valued functions of bounded variation.
\end{theorem}
The proof can be found in page 35 in \cite{lyons2007differential}. The proof is based on the Fubini theorem, and it is to show that for any $e^{*}, f^{*} \in T(E^{*})$, such that for all $\mathbf{a} \in S(\mathcal{V}^{p}([0, T], E)$,
\begin{equation}
    e^{*}(\mathbf{a}) f^{*}(\mathbf{a}) = (e^{*} \shuffle f^{*})(\mathbf{a})
\end{equation}
\subsection{Universality of the signature}
Any functional on the path can be rewritten as a function on the signature based on the uniqueness of the signature (Theorem \ref{UniquenessOfSig}). The signature of the path has the universality property, i.e. that any continuous functional on the signature can be well approximated by linear functionals on the signature (Theorem \ref{SigApproximatinTheorem})\cite{levin2013learning}.

\begin{theorem}[Signature Approximation Theorem]\label{SigApproximatinTheorem}
	Suppose $f: S_{1} \rightarrow \mathbb{R}$ is a continuous function, where $S_{1}$ is a compact subset of $S(BV(J,E))$\footnote{ $S(BV(J,E))$ denotes the range of the signature for $x \in BV(J,E)$.}. Then $\forall \varepsilon >0$, there exists a linear functional $L \in T((E))^{*}$ such that
	\begin{eqnarray}
	\sup_{a \in S_{1}}\vert \vert f(a) - L(a) \vert \vert \leq \varepsilon. 
	\end{eqnarray}
\end{theorem}
\begin{proof}
It can be proved by the shuffle product property of the signature and the Stone-Weierstrass Theorem. 
\end{proof}

\section{The log-signature of a path}
Before introducing the log-signature, we define the Lie algebra, which the log-signature takes value in.
\subsection{Lie algebra and Lie series}
If $F_{1}$ and $F_{2}$ are two linear subspaces of $T((E))$, let us denote by $[F_{1}, F_{2}]$ the linear span of all the elements of the form $[a,b]$, where $a \in F_{1}$ and $b \in F_{2}$.
Consider the sequence $(L_{n})_{n \geq 0}$ of subspaces of $T((E))$ defined recursively as follows:
\begin{equation}
    L_{0} = 0;  \forall n \geq 1,  L_{n} = [E, L_{n-1}]. 
\end{equation}
\begin{definition}\label{def_lie_seires}
The space of Lie formal series over $E$, denoted as $\mathcal{L}((E))$ is defined as the following subspace of $T((E))$:
\begin{eqnarray}
\mathcal{L}((E)) = \{l = (l_{0}, \cdots, l_{n}, \cdots) \vert \forall n \geq 0, l_{n} \in L_{n}\}.
\end{eqnarray}
\end{definition}

\begin{theorem}[Theorem 2.23 \cite{lyons2007differential}]
Let $X \in BV(J, E)$. Then the log-signature of $X$ is a Lie series in $\mathcal{L}((E))$.
\end{theorem}

\subsection{The bijection between the signature and log-signature}
Similar to the way of defining the logarithm of a tensor series, we have the exponential mapping of the element in $T((E))$ defined in a power series form.
\begin{definition}[Exponential map]
	Let $a = (a_{0}, a_{1}, \cdots) \in T((E))$. Define the exponential map denoted by $\exp$ as follows:
	\begin{eqnarray}
	\exp (a) = \sum_{n = 0}^{\infty}\frac{a^{\otimes n}}{n!}.
	\end{eqnarray}
\end{definition}
\begin{lemma}
The inverse of the logarithm on the domain $\{a \in T((E)) | a_{0} \neq 0\}$ is the exponential map.
\end{lemma}

\begin{theorem}\label{dim_hall}
The dimension of the space of the truncated log signature of $d$-dimensional path up to degree $n$ over $d$ letters is given by:

$$\mathcal{DL}_n = \frac{1}{n} \sum_{d|n}\mu(d)q^{n|d}$$

where $\mu$ is the Mobius function, which maps $n$ to
\begin{eqnarray*}
 \left\{
    \begin{array}{ll}
        0, & \mbox{if $n$ has one or more repeated prime factors} \\
        1, & \mbox{if $n=1$} \\
        (-1)^k & \mbox{if $n$ is the product of k distinct prime numbers}
    \end{array}
\right.
\end{eqnarray*}
\end{theorem}
The proof can be found in  Corollary 4.14 p. 96 of \cite{reutenauer2003free}.

\subsection{Calculation of the log-signature}
Let's start with a linear path. The log signature of a linear path $X_{J}$ is nothing else, but the increment of the path $X_{T} - X_{S}$. \\

The Baker-Cambpell-Hausdorff (BCH) formula gives a general method to compute the log-signature of the concatenation of two paths, which uses the multiplicativity of the signature and the free Lie algebra. It provides a way to compute the log-signature of a piecewise linear path by induction. 
\begin{theorem}
For any $S_1, S_2 \in \mathcal{L}((E))$
\begin{eqnarray}
Z = log(e^{S_1}e^{S_2})=  \sum_{\substack{n\geq 1 \\ p_1,...,p_n \geq 0\\q_1,....q_n \geq 0 \\ p_i+q_i >0}}\frac{(-1)^{n+1}}{n} \frac{1}{p_1!q_1!...p_n!q_n!}r(S_1^{p_1}S_2^{q_1}...S_1^{p_n}S_2^{q_n})
\end{eqnarray}
where $r:A^* \rightarrow A^*$ is the right-Lie-bracketing operator, such that for any word $w=a_1...a_n$ 
$$r(w)=[a_1,...,[a_{n-1}, a_n]...].$$
This version of BCH is sometimes called the Dynkin's formula. 
\end{theorem}

\begin{proof}
See remark of appendix 3.5.4 p. 81 in \cite{reutenauer2003free}.
\end{proof}
\subsection{Uniqueness of the log-signature}
Like the signature, the log-signature has the uniqueness stated in the following theorem.
\begin{theorem}[Uniqueness of the log-signature]\label{UniquenessOflogSig}
	Let $X \in BV(J,E)$ . Then $lS(X)$ determines $X$ up to the tree-like equivalence defined in Definition \ref{def_tree_like}. 
\end{theorem}
Theorem \ref{UniquenessOflogSig} shows that the signature of the path can recover the path trajectory under a mild condition. 
\begin{lemma}\label{Lemma_mono_logsig}
A simple sufficient condition for the uniqueness of the log-signature of a path of finite length is that one component of $X$ is monotone. 
\end{lemma}

\section{Comparison of the Signature and Log-signature}
Both the signature and log-signature take the functional view on discrete time series data, which allows a unified way to treat time series of variable length and missing data. For example, we chose one pen-digit data of length 53 and simulate 1000 samples of modified pen trajectories by dropping at most 16 points from it, to mimic the missing data of variable length case (See one example in Figure \ref{missing_data9}). Figure \ref{sig_logsig_comparison_missing_data} shows that the mean absolute relative error of the signature and log-signature of the missing data is no more than $6\%$. Besides, the log-signature feature is more robust against missing data, and it is of lower dimension compared with the signature feature.

\begin{figure}[!ht]
	\centering
	\includegraphics[width= 0.65\textwidth]{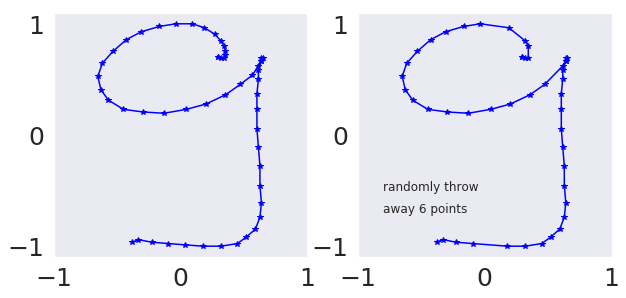}
	\caption{(Left) The chosen pen trajectory of digit 9. (Right) The simulated path by randomly dropping at most 16 points of the pen trajectory on the left.
	}\label{missing_data9}
\end{figure}
\begin{figure}[!ht]
	\centering
	\includegraphics[width= 0.65\textwidth]{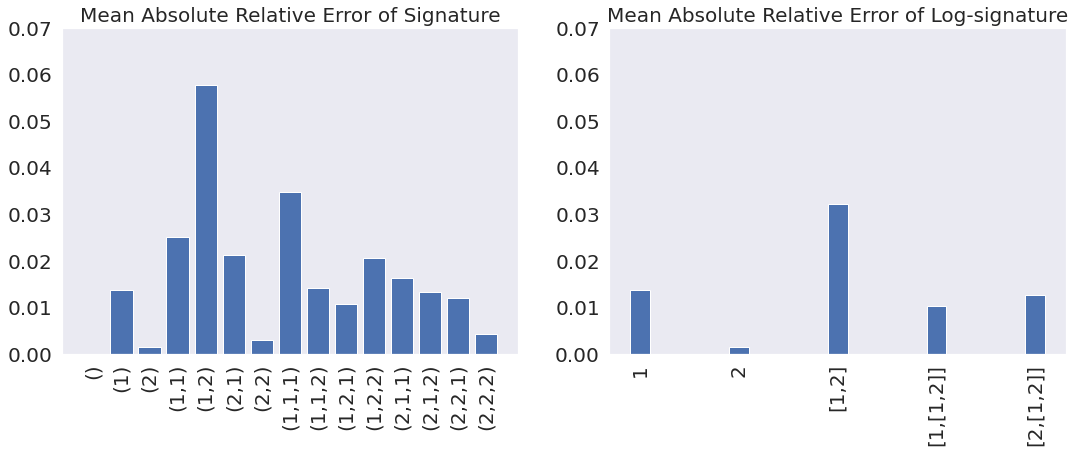}
	\caption{Signature and Log-Signature Comparison for the missing data case.
	}\label{sig_logsig_comparison_missing_data}
\end{figure}
\begin{figure}[!ht]
	\centering
	\includegraphics[width= 0.65\textwidth]{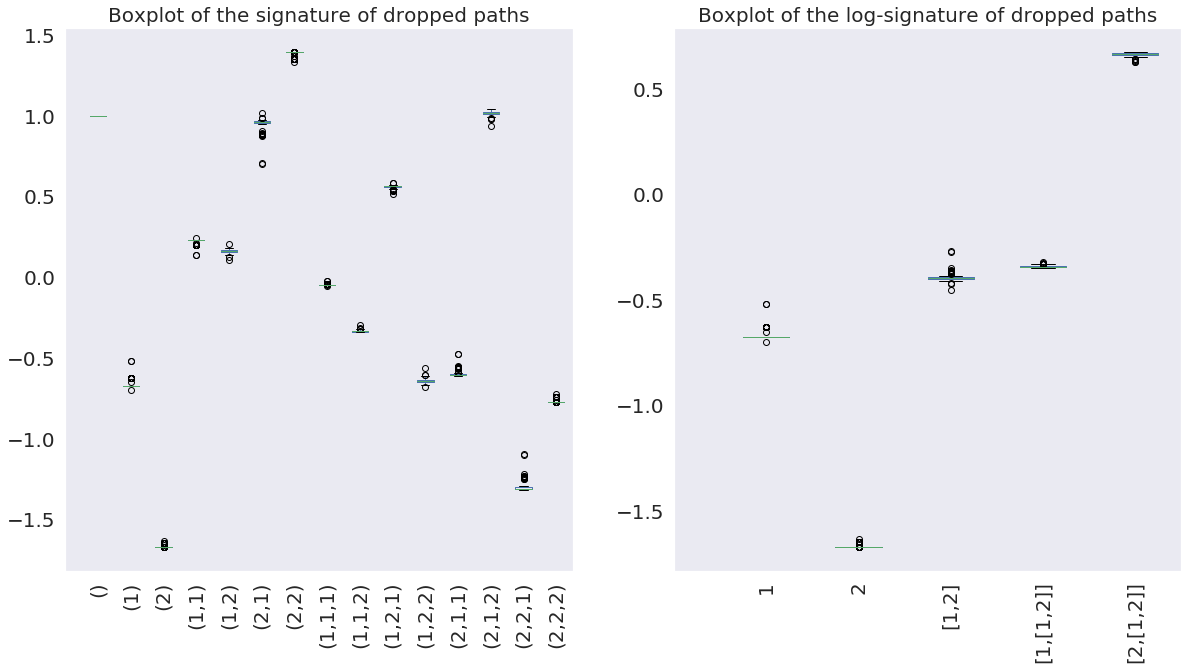}
	\caption{Signature and Log-Signature Comparison for the missing data case.
	}\label{boxplot_Sig_Logsig}
\end{figure}
\section{Backpropagation of the Log-Signature Layer}
In this section, we provide the detailed derivation of backpropagation through time in our paper. Following the notations in Section 4 of our paper, $l_{k}:=l_{k}^{M}$ denotes the log-signatures of a path $x^{\hat{\mathcal{D}}}$ over time partition $[u_{k}, u_{k+1}]$ of degree $M$. Let us consider the derivative of the scalar function $F$ on $(l_{k})_{k = 1}^{N}$ with respect to path $x^{\hat{\mathcal{D}}}$, given the derivatives of $F$ with respect to $(l_{k})_{k = 1}^{N}$. W.l.o.g assume that $\hat{\mathcal{D}} \subset \mathcal{D}$. By the Chain rule, it holds that
\begin{eqnarray}\label{BP_1}
\frac{\partial F((l_{1}, \cdots, l_{N}))}{\partial x_{t_{i}}} = \sum_{k = 1}^{N}\frac{\partial F(l_{1}, \cdots, l_{N})}{\partial l_{k}}\frac{\partial l_{k}}{\partial x_{t_{i}}}.
\end{eqnarray}
where $k \in \{1, \cdots, N\}$ and $i \in \{0, 1, \cdots, n\}$.

If $t_{i} \notin [u_{k-1}, u_{k}]$, $\frac{\partial l_{k}}{\partial x_{t_{i}}} = 0;$ otherwise $\frac{\partial l_{k}}{\partial x_{t_{i}}}$ is the derivative of the single log-signature $l_{k}$ with respect to path $x_{u_{k-1}, u_{k}}$  where $t_{i} \in \mathcal{D} \cap [u_{k-1}, u_{k}]$. The log signature $lS(x^{\hat{\mathcal{D}}})$ with respect to $x_{t_{i}}$ is proved differentiable and the algorithm of computing the derivatives is given in \cite{ReizensteinIhesis2018}, denoted by $\triangledown_{x_{t_{i}}} LS(x^{\hat{\mathcal{D}}})$. This is a special case for our log-signature layer when $N = 1$. In general, for any $N \in \mathbb{Z}^{+}$, it holds that $\forall  i \in \{0, 1, \cdots, n\}$ and $k \in \{1, \cdots, N\}$,\footnote{Our implementation of Log-Signature layer in Tensorflow is implemented using the iisignature python package \cite{ReizensteinIisignature2018} and the customized layer in Keras. We implement the pytorch version of the Log-Signature layer using signatory package with some modification to handle time series of variable length. }%In iisignature python package \cite{ReizensteinIisignature2018}, logsigbackprop(deriv, path, s, Method = None) returns the derivatives of some scalar function $F$ with respect to path, given the derivatives of $F$ with respect to logsig(path, s, methods). Our implementation of Log-Signature layer uses logigbackprop() provided in iisignature and the cunstomized layer in Keras.}
\begin{eqnarray}\label{logsigseq_eqn}
\frac{\partial l_{k}}{\partial x_{t_{i}}} = \mathbf{1}_{t_{i} \in [u_{k-1}, u_{k}]} \triangledown_{x_{t_{i}}} LS(x_{u_{k-1}, u_{k}}).
\end{eqnarray}
 The backpropagation of the Log-Signature Layer can be implemented using Equation \eqref{BP_1} and
\eqref{logsigseq_eqn}.
\section{Other Path Transformation Layers}
In this section, we introduce two other useful Path Transformation Layers which are often accompanied with the Embedding Layer in the PT-Logsig-RNN to improve the performance. 

\textbf{Accumulative Layer}. Accumulative Layer (AL) maps the input sequence $(X_{t_{i}})_{i = 1}^{n}$ to its partial sum sequence $Y_{t_{i}}$, where $Y_{t_{i}} = \sum_{j = 1}^{i}X_{t_{j}}$, and $i=1, \cdots, n$. One advantage of using the Accumulative Layer along with Log-Signature Layer is to extract the quadratic variation and other higher order statistics of an input path $X$ effectively \cite{ni2015multi}.  

\textbf{Time-incorporated Layer}. Time-incorporated Layer (TL) is to add the time dimension to the input $(X_{t_{i}})_{i = 1}^{n}$; in formula, the output is $(t_{i}, X_{t_{i}})_{i = 1}^{n}$. As speed information is informative in most of SHAR tasks, e.g. distinguish running and walking, the log-signature of the time-incorporated transformation of a path can preserve such information.

\section{Numerical Experiments}
Network Architecture of PT-Logsig-RNN on Chalearn2013 data can be found in Table \ref{gesturearche}.
\begin{table}[!h]
\centering
\scalebox{1}{
\begin{tabular}{l|l|l|l}
\hline
\multicolumn{2}{c|}{\textbf{Layer}} & \textbf{Output shape} & \textbf{Description}\\
\hline
\multicolumn{2}{c|}{Input} & $(n, 20, 3)$ & Input with $n$ time steps \\
\cline{1-2}
&Conv2D & $(n, 20, 32)$ & Kernel size=$1\times 1\times 32$ \\
Embedding&Conv2D & $(n, 20, 16)$ & Kernel size=$3\times 1\times 16$\\
$~~~~$Layer&Flatten & $(n,320) $& Flatten spatial dimensions\\
&Conv1D & $(n, 30)$ & Kernel size=$1\times 304\times 30$\\
\cline{1-2}
\multicolumn{2}{c|}{Accumulative Layer (AL)} & $(n,30)$ & Partial sum on the series\\
\multicolumn{2}{c|}{Time-incorporated Layer (TL)} & $(n,31)$ &Add temporal dimension\\
\multicolumn{2}{c|}{Logsig Layer} & $(N, 496)$ & $M=2, N=4$ after tuning\\
\multicolumn{2}{c|}{Add Starting Points} & $(N,527)$ & Starting points of TL output\\
\multicolumn{2}{c|}{LSTM} & $(N,128)$ & Return sequential output \\
\multicolumn{2}{c|}{Output} & $20$ &\\
\hline
\end{tabular}}
\caption{Architecture of the PT-Logsig-RNN on Chalearn 2013}
\label{gesturearche}
\end{table}
\bibliographystyle{unsrtnat}
\bibliography{egbib}  

\begin{thebibliography}{61}
\providecommand{\natexlab}[1]{#1}
\providecommand{\url}[1]{\texttt{#1}}
\expandafter\ifx\csname urlstyle\endcsname\relax
  \providecommand{\doi}[1]{doi: #1}\else
  \providecommand{\doi}{doi: \begingroup \urlstyle{rm}\Url}\fi

\bibitem[Lo~Presti and La~Cascia(2016)]{lo20163d}
Liliana Lo~Presti and Marco La~Cascia.
\newblock 3d skeleton-based human action classification.
\newblock \emph{Pattern Recognition}, 53\penalty0 (C):\penalty0 130--147, 2016.

\bibitem[Wang et~al.(2019{\natexlab{a}})Wang, Huynh, and
  Koniusz]{wang2019comparative}
Lei Wang, Du~Q Huynh, and Piotr Koniusz.
\newblock A comparative review of recent kinect-based action recognition
  algorithms.
\newblock \emph{IEEE Transactions on Image Processing}, 29:\penalty0 15--28,
  2019{\natexlab{a}}.

\bibitem[Ren et~al.(2020)Ren, Liu, Ding, and Liu]{ren2020survey}
Bin Ren, Mengyuan Liu, Runwei Ding, and Hong Liu.
\newblock A survey on 3d skeleton-based action recognition using learning
  method.
\newblock \emph{arXiv preprint arXiv:2002.05907}, 2020.

\bibitem[Lev et~al.(2016)Lev, Sadeh, Klein, and Wolf]{lev2016rnn}
Guy Lev, Gil Sadeh, Benjamin Klein, and Lior Wolf.
\newblock Rnn fisher vectors for action recognition and image annotation.
\newblock In \emph{European Conference on Computer Vision}, pages 833--850.
  Springer, 2016.

\bibitem[Liu et~al.(2016)Liu, Shahroudy, Xu, and
  Wang]{10.1007/978-3-319-46487-9_50}
Jun Liu, Amir Shahroudy, Dong Xu, and Gang Wang.
\newblock Spatio-temporal lstm with trust gates for 3d human action
  recognition.
\newblock In \emph{ECCV}, pages 816--833, 2016.
\newblock ISBN 978-3-319-46487-9.

\bibitem[Liu et~al.(2018)Liu, Wang, yu~Duan, Abdiyeva, and
  Kot]{Liu2018SkeletonBasedHA}
Jun Liu, Guanghui Wang, Ling yu~Duan, Kamila Abdiyeva, and Alex~ChiChung Kot.
\newblock Skeleton-based human action recognition with global context-aware
  attention lstm networks.
\newblock \emph{IEEE TIP,}, 27:\penalty0 1586--1599, 2018.

\bibitem[Shahroudy et~al.(2016)]{nturgb}
Amir Shahroudy et~al.
\newblock Ntu rgb+d: A large scale dataset for 3d human activity analysis.
\newblock In \emph{CVPR}, 06 2016.
\newblock \doi{10.1109/CVPR.2016.115}.

\bibitem[Wang and Wang(2017)]{Wang2017ModelingTD}
Hongsong Wang and Liang Wang.
\newblock Modeling temporal dynamics and spatial configurations of actions
  using two-stream recurrent neural networks.
\newblock \emph{CVPR,}, pages 3633--3642, 2017.

\bibitem[Caetano et~al.(2019)Caetano, Sena, Br{\'e}mond, Dos~Santos, and
  Schwartz]{caetano2019skelemotion}
Carlos Caetano, Jessica Sena, Fran{\c{c}}ois Br{\'e}mond, Jefersson~A
  Dos~Santos, and William~Robson Schwartz.
\newblock Skelemotion: A new representation of skeleton joint sequences based
  on motion information for 3d action recognition.
\newblock In \emph{2019 16th IEEE International Conference on Advanced Video
  and Signal Based Surveillance (AVSS)}, pages 1--8. IEEE, 2019.

\bibitem[Ch{\'e}ron et~al.(2015)Ch{\'e}ron, Laptev, and Schmid]{cheron2015p}
Guilhem Ch{\'e}ron, Ivan Laptev, and Cordelia Schmid.
\newblock P-cnn: Pose-based cnn features for action recognition.
\newblock In \emph{Proceedings of the IEEE international conference on computer
  vision}, pages 3218--3226, 2015.

\bibitem[{Ke} et~al.(2018){Ke}, {Bennamoun}, {An}, {Sohel}, and
  {Boussaid}]{8306456}
Q.~{Ke}, M.~{Bennamoun}, S.~{An}, F.~{Sohel}, and F.~{Boussaid}.
\newblock Learning clip representations for skeleton-based 3d action
  recognition.
\newblock \emph{IEEE TIP,}, 27\penalty0 (6):\penalty0 2842--2855, June 2018.
\newblock ISSN 1057-7149.
\newblock \doi{10.1109/TIP.2018.2812099}.

\bibitem[Li et~al.(2018)Li, Zhong, Xie, and Pu]{li2018co}
Chao Li, Qiaoyong Zhong, Di~Xie, and Shiliang Pu.
\newblock Co-occurrence feature learning from skeleton data for action
  recognition and detection with hierarchical aggregation.
\newblock In \emph{Proceedings of the 27th International Joint Conference on
  Artificial Intelligence}, pages 786--792, 2018.

\bibitem[Liu et~al.(2019{\natexlab{a}})Liu, Shahroudy, Wang, Duan, and
  C.~Kot]{liufsnet}
Jun Liu, Amir Shahroudy, Gang Wang, Ling-Yu Duan, and Alex C.~Kot.
\newblock Skeleton-based online action prediction using scale selection
  network.
\newblock \emph{IEEE TPAMI}, 02 2019{\natexlab{a}}.
\newblock \doi{10.1109/TPAMI.2019.2898954}.

\bibitem[Liu and Yuan(2018)]{liu2018recognizing}
Mengyuan Liu and Junsong Yuan.
\newblock Recognizing human actions as the evolution of pose estimation maps.
\newblock In \emph{CVPR}, pages 1159--1168, 2018.

\bibitem[Li et~al.(2019{\natexlab{a}})Li, Chen, Chen, Zhang, Wang, and
  Tian]{li2019actional}
Maosen Li, Siheng Chen, Xu~Chen, Ya~Zhang, Yanfeng Wang, and Qi~Tian.
\newblock Actional-structural graph convolutional networks for skeleton-based
  action recognition.
\newblock In \emph{Proceedings of the IEEE Conference on Computer Vision and
  Pattern Recognition}, pages 3595--3603, 2019{\natexlab{a}}.

\bibitem[Shi et~al.(2019)Shi, Zhang, Cheng, and Lu]{shi2019skeleton}
Lei Shi, Yifan Zhang, Jian Cheng, and Hanqing Lu.
\newblock Skeleton-based action recognition with directed graph neural
  networks.
\newblock In \emph{Proceedings of the IEEE Conference on Computer Vision and
  Pattern Recognition}, pages 7912--7921, 2019.

\bibitem[Fan et~al.(2016)Fan, Zha, and Tian]{fan2016action}
Jiayi Fan, Zhengjun Zha, and Xinmei Tian.
\newblock Action recognition with novel high-level pose features.
\newblock In \emph{2016 IEEE International Conference on Multimedia \& Expo
  Workshops (ICMEW)}, pages 1--6. IEEE, 2016.

\bibitem[Jhuang et~al.(2013)Jhuang, Gall, Zuffi, Schmid, and
  Black]{jhuang2013towards}
Hueihan Jhuang, Juergen Gall, Silvia Zuffi, Cordelia Schmid, and Michael~J
  Black.
\newblock Towards understanding action recognition.
\newblock In \emph{Proceedings of the IEEE international conference on computer
  vision}, pages 3192--3199, 2013.

\bibitem[Yang and Tian(2014)]{yang2014effective}
Xiaodong Yang and YingLi Tian.
\newblock Effective 3d action recognition using eigenjoints.
\newblock \emph{Journal of Visual Communication and Image Representation},
  25\penalty0 (1):\penalty0 2--11, 2014.

\bibitem[Evangelidis et~al.(2014)Evangelidis, Singh, and
  Horaud]{evangelidis2014skeletal}
Georgios Evangelidis, Gurkirt Singh, and Radu Horaud.
\newblock Skeletal quads: Human action recognition using joint quadruples.
\newblock In \emph{2014 22nd International Conference on Pattern Recognition},
  pages 4513--4518. IEEE, 2014.

\bibitem[Li and Leung(2016)]{li2016multiview}
Meng Li and Howard Leung.
\newblock Multiview skeletal interaction recognition using active joint
  interaction graph.
\newblock \emph{IEEE Transactions on Multimedia}, 18\penalty0 (11):\penalty0
  2293--2302, 2016.

\bibitem[Shahroudy et~al.(2015)Shahroudy, Ng, Yang, and
  Wang]{shahroudy2015multimodal}
Amir Shahroudy, Tian-Tsong Ng, Qingxiong Yang, and Gang Wang.
\newblock Multimodal multipart learning for action recognition in depth videos.
\newblock \emph{IEEE transactions on pattern analysis and machine
  intelligence}, 38\penalty0 (10):\penalty0 2123--2129, 2015.

\bibitem[Ke et~al.(2017)Ke, An, Bennamoun, Sohel, and
  Boussaid]{ke2017skeletonnet}
Qiuhong Ke, Senjian An, Mohammed Bennamoun, Ferdous Sohel, and Farid Boussaid.
\newblock Skeletonnet: Mining deep part features for 3-d action recognition.
\newblock \emph{IEEE signal processing letters}, 24\penalty0 (6):\penalty0
  731--735, 2017.

\bibitem[{Liu} et~al.(2020){Liu}, {Zhang}, {Chen}, {Wang}, and
  {Ouyang}]{msg3d_liu}
Z.~{Liu}, H.~{Zhang}, Z.~{Chen}, Z.~{Wang}, and W.~{Ouyang}.
\newblock Disentangling and unifying graph convolutions for skeleton-based
  action recognition.
\newblock In \emph{2020 IEEE/CVF Conference on Computer Vision and Pattern
  Recognition (CVPR)}, pages 140--149, 2020.
\newblock \doi{10.1109/CVPR42600.2020.00022}.

\bibitem[Lyons(1998)]{lyons1998differential}
Terry~J Lyons.
\newblock Differential equations driven by rough signals.
\newblock \emph{Revista Matem{\'a}tica Iberoamericana,}, 14\penalty0
  (2):\penalty0 215--310, 1998.

\bibitem[Levin et~al.(2013)Levin, Lyons, and Ni]{levin2013learning}
Daniel Levin, Terry Lyons, and Hao Ni.
\newblock Learning from the past, predicting the statistics for the future,
  learning an evolving system.
\newblock \emph{arXiv preprint arXiv:1309.0260}, 2013.

\bibitem[Gyurk{\'o} et~al.(2013)Gyurk{\'o}, Lyons, Kontkowski, and
  Field]{gyurko2013extracting}
Lajos~Gergely Gyurk{\'o}, Terry Lyons, Mark Kontkowski, and Jonathan Field.
\newblock Extracting information from the signature of a financial data stream.
\newblock \emph{arXiv preprint arXiv:1307.7244}, 2013.

\bibitem[Lyons et~al.(2014)Lyons, Ni, and Oberhauser]{lyons2014feature}
Terry Lyons, Hao Ni, and Harald Oberhauser.
\newblock A feature set for streams and an application to high-frequency
  financial tick data.
\newblock In \emph{ACM International Conference on Big Data Science and
  Computing}, page~5, 2014.

\bibitem[Diehl(2013)]{diehl2013rotation}
Joscha Diehl.
\newblock Rotation invariants of two dimensional curves based on iterated
  integrals.
\newblock \emph{arXiv preprint arXiv:1305.6883}, 2013.

\bibitem[Graham(2013)]{graham2013sparse}
Benjamin Graham.
\newblock Sparse arrays of signatures for online character recognition.
\newblock \emph{arXiv preprint arXiv:1308.0371}, 2013.

\bibitem[Yang et~al.(2016{\natexlab{a}})Yang, Jin, Tao, Xie, and
  Feng]{yang2016dropsample}
Weixin Yang, Lianwen Jin, Dacheng Tao, Zecheng Xie, and Ziyong Feng.
\newblock Dropsample: A new training method to enhance deep convolutional
  neural networks for large-scale unconstrained handwritten chinese character
  recognition.
\newblock \emph{Pattern Recognition}, 58:\penalty0 190--203,
  2016{\natexlab{a}}.

\bibitem[Xie et~al.(2018)Xie, Sun, Jin, Ni, and Lyons]{xie2018learning}
Zecheng Xie, Zenghui Sun, Lianwen Jin, Hao Ni, and Terry Lyons.
\newblock Learning spatial-semantic context with fully convolutional recurrent
  network for online handwritten chinese text recognition.
\newblock \emph{IEEE TPAMI,}, 40\penalty0 (8):\penalty0 1903--1917, 2018.

\bibitem[Yang et~al.(2016{\natexlab{b}})Yang, Jin, and
  Liu]{yang2016deepwriterid}
Weixin Yang, Lianwen Jin, and Manfei Liu.
\newblock Deepwriterid: An end-to-end online text-independent writer
  identification system.
\newblock \emph{IEEE Intelligent Systems}, 31\penalty0 (2):\penalty0 45--53,
  2016{\natexlab{b}}.

\bibitem[Lai et~al.(2017)Lai, Jin, and Yang]{lai2017online}
Songxuan Lai, Lianwen Jin, and Weixin Yang.
\newblock Online signature verification using recurrent neural network and
  length-normalized path signature descriptor.
\newblock In \emph{2017 14th IAPR International Conference on Document Analysis
  and Recognition (ICDAR)}, volume~1, pages 400--405. IEEE, 2017.

\bibitem[Arribas et~al.(2018)Arribas, Goodwin, Geddes, Lyons, and
  Saunders]{arribas2018signature}
Imanol~Perez Arribas, Guy~M Goodwin, John~R Geddes, Terry Lyons, and Kate~EA
  Saunders.
\newblock A signature-based machine learning model for distinguishing bipolar
  disorder and borderline personality disorder.
\newblock \emph{Translational psychiatry}, 8\penalty0 (1):\penalty0 1--7, 2018.

\bibitem[Wang et~al.(2019{\natexlab{b}})Wang, Liakata, Ni, Lyons,
  Nevado-Holgado, and Saunders]{wang2019path}
Bo~Wang, Maria Liakata, Hao Ni, Terry Lyons, Alejo~J Nevado-Holgado, and Kate
  Saunders.
\newblock A path signature approach for speech emotion recognition.
\newblock In \emph{Interspeech 2019}, pages 1661--1665. ISCA,
  2019{\natexlab{b}}.

\bibitem[Ahmad et~al.(2019)Ahmad, Jin, Feng, and Tang]{ahmad2019human}
Tasweer Ahmad, Lianwen Jin, Jialuo Feng, and Guozhi Tang.
\newblock Human action recognition in unconstrained trimmed videos using
  residual attention network and joints path signature.
\newblock \emph{IEEE Access}, 7:\penalty0 121212--121222, 2019.

\bibitem[Kir{\'a}ly and Oberhauser(2016)]{kiraly2016kernels}
Franz~J Kir{\'a}ly and Harald Oberhauser.
\newblock Kernels for sequentially ordered data.
\newblock \emph{arXiv preprint arXiv:1601.08169}, 2016.

\bibitem[Li et~al.(2019{\natexlab{b}})Li, Zhang, Liao, Jin, and
  Yang]{li2019skeleton}
Chenyang Li, Xin Zhang, Lufan Liao, Lianwen Jin, and Weixin Yang.
\newblock Skeleton-based gesture recognition using several fully connected
  layers with path signature features and temporal transformer module.
\newblock In \emph{AAAI}, pages 8585--8593, 2019{\natexlab{b}}.

\bibitem[Yang et~al.(2017)Yang, Lyons, Ni, Schmid, Jin, and
  Chang]{yang2017leveraging}
Weixin Yang, Terry Lyons, Hao Ni, Cordelia Schmid, Lianwen Jin, and Jiawei
  Chang.
\newblock Leveraging the path signature for skeleton-based human action
  recognition.
\newblock \emph{arXiv preprint arXiv:1707.03993}, 2017.

\bibitem[Li et~al.(2017)Li, Zhang, and Jin]{li2017lpsnet}
Chenyang Li, Xin Zhang, and Lianwen Jin.
\newblock Lpsnet: A novel log path signature feature based hand gesture
  recognition framework.
\newblock In \emph{2017 IEEE International Conference on Computer Vision
  Workshops (ICCVW)}, pages 631--639, 2017.
\newblock \doi{10.1109/ICCVW.2017.80}.

\bibitem[Kidger et~al.(2021)Kidger, Morrill, Foster, and
  Lyons]{morrill2021ICML}
Patrick Kidger, James Morrill, James Foster, and Terry Lyons.
\newblock Neural controlled differential equations for irregular time series.
\newblock In \emph{International Conference on Machine Learning (ICML)}, 2021.

\bibitem[Kidger et~al.(2019)Kidger, Bonnier, Perez~Arribas, Salvi, and
  Lyons]{kidger2019deep}
Patrick Kidger, Patric Bonnier, Imanol Perez~Arribas, Cristopher Salvi, and
  Terry Lyons.
\newblock Deep signature transforms.
\newblock In \emph{Advances in Neural Information Processing Systems},
  volume~32, 2019.

\bibitem[Chevyrev and Kormilitzin(2016)]{chevyrev2016primer}
Ilya Chevyrev and Andrey Kormilitzin.
\newblock A primer on the signature method in machine learning.
\newblock \emph{arXiv preprint arXiv:1603.03788}, 2016.

\bibitem[Lyons et~al.()Lyons, Chafai, Buckley, Gyurko, and Janssen]{esig}
Terry Lyons, Djalil Chafai, Stephen Buckley, Greg Gyurko, and Arend Janssen.
\newblock Esig on pypi derived from coropa: Computational rough paths software
  library.
\newblock
  \url{https://github.com/datasig-ac-uk/esig},\url{https://github.com/terrylyons/libalgebra},
  \url{http://coropa.sourceforge.net/}.
\newblock [Online].

\bibitem[Jeremy and Benjamin.(2018)]{ReizensteinIisignature2018}
Reizenstein Jeremy and Graham Benjamin.
\newblock The iisignature library: efficient calculation of iterated-integral
  signatures and log signatures.
\newblock \emph{arXiv preprint arXiv:1802.08252.}, 2018.

\bibitem[Kidger and Lyons(2020)]{kidger2020signatory}
Patrick Kidger and Terry Lyons.
\newblock Signatory: differentiable computations of the signature and
  logsignature transforms, on both cpu and gpu.
\newblock \emph{arXiv preprint arXiv:2001.00706}, 2020.

\bibitem[Hambly and Lyons(2010)]{UniquenessOfSignature}
B.M. Hambly and Terry Lyons.
\newblock Uniqueness for the signature of a path of bounded variation and the
  reduced path group.
\newblock \emph{Annals of Mathematics,}, 171\penalty0 (1):\penalty0 109--167,
  2010.

\bibitem[Reizenstein and Graham(2018)]{reizenstein2018iisignature}
Jeremy Reizenstein and Benjamin Graham.
\newblock The iisignature library: efficient calculation of iterated-integral
  signatures and log signatures.
\newblock \emph{arXiv preprint arXiv:1802.08252}, 2018.

\bibitem[Escalera et~al.(2013)Escalera, Gonz{\`a}lez, Bar{\'o}, Reyes, Lopes,
  Guyon, Athitsos, and Escalante]{Escalera2013MultimodalGR}
Sergio Escalera, Jordi Gonz{\`a}lez, Xavier Bar{\'o}, Miguel Reyes, Oscar
  Lopes, Isabelle Guyon, Vassilis Athitsos, and Hugo~Jair Escalante.
\newblock Multi-modal gesture recognition challenge 2013: dataset and results.
\newblock In \emph{ICMI}, 2013.

\bibitem[Liu et~al.(2019{\natexlab{b}})Liu, Shahroudy, Perez, Wang, Duan, and
  Kot]{Liu_2019_NTURGBD120}
Jun Liu, Amir Shahroudy, Mauricio Perez, Gang Wang, Ling-Yu Duan, and Alex~C.
  Kot.
\newblock Ntu rgb+d 120: A large-scale benchmark for 3d human activity
  understanding.
\newblock \emph{IEEE TPAMI}, 2019{\natexlab{b}}.
\newblock \doi{10.1109/TPAMI.2019.2916873}.

\bibitem[Liao et~al.(2019)Liao, Zhang, and Li]{liao2019multi}
Lufan Liao, Xin Zhang, and Chenyang Li.
\newblock Multi-path convolutional neural network based on rectangular kernel
  with path signature features for gesture recognition.
\newblock In \emph{2019 IEEE Visual Communications and Image Processing
  (VCIP)}, pages 1--4. IEEE, 2019.

\bibitem[{Liu} et~al.(2018){Liu}, {Shahroudy}, {Xu}, {Kot}, and
  {Wang}]{8101019}
J.~{Liu}, A.~{Shahroudy}, D.~{Xu}, A.~C. {Kot}, and G.~{Wang}.
\newblock Skeleton-based action recognition using spatio-temporal lstm network
  with trust gates.
\newblock \emph{IEEE TPAMI,}, 40\penalty0 (12):\penalty0 3007--3021, Dec 2018.
\newblock ISSN 0162-8828.
\newblock \doi{10.1109/TPAMI.2017.2771306}.

\bibitem[Song et~al.(2021)Song, Zhang, Shan, and Wang]{ra_gcn_song}
Yi-Fan Song, Zhang Zhang, Caifeng Shan, and Liang Wang.
\newblock Richly activated graph convolutional network for robust
  skeleton-based action recognition.
\newblock \emph{IEEE Transactions on Circuits and Systems for Video
  Technology}, 31\penalty0 (5):\penalty0 1915--1925, 2021.
\newblock \doi{10.1109/TCSVT.2020.3015051}.

\bibitem[Cheng et~al.(2020)Cheng, Zhang, He, Chen, Cheng, and
  Lu]{shiftgcn_Cheng}
Ke~Cheng, Yifan Zhang, Xiangyu He, Weihan Chen, Jian Cheng, and Hanqing Lu.
\newblock Skeleton-based action recognition with shift graph convolutional
  network.
\newblock In \emph{Proceedings of the IEEE/CVF Conference on Computer Vision
  and Pattern Recognition (CVPR)}, June 2020.

\bibitem[Song et~al.(2020)Song, Zhang, Shan, and Wang]{pa_res_song}
Yi-Fan Song, Zhang Zhang, Caifeng Shan, and Liang Wang.
\newblock Stronger, faster and more explainable: A graph convolutional baseline
  for skeleton-based action recognition.
\newblock In \emph{Proceedings of the 28th ACM International Conference on
  Multimedia}, MM '20, page 1625–1633, New York, NY, USA, 2020. Association
  for Computing Machinery.
\newblock ISBN 9781450379885.
\newblock \doi{10.1145/3394171.3413802}.
\newblock URL \url{https://doi.org/10.1145/3394171.3413802}.

\bibitem[Mao et~al.(2019)Mao, Liu, Salzmann, and Li]{mao2019learning}
Wei Mao, Miaomiao Liu, Mathieu Salzmann, and Hongdong Li.
\newblock Learning trajectory dependencies for human motion prediction.
\newblock In \emph{Proceedings of the IEEE/CVF International Conference on
  Computer Vision}, pages 9489--9497, 2019.

\bibitem[Lyons et~al.(2007)Lyons, Caruana, and L{\'e}vy]{lyons2007differential}
Terry~J Lyons, Michael Caruana, and Thierry L{\'e}vy.
\newblock \emph{Differential equations driven by rough paths}.
\newblock Springer, 2007.

\bibitem[Reutenauer(2003)]{reutenauer2003free}
Christophe Reutenauer.
\newblock Free lie algebras.
\newblock In \emph{Handbook of algebra}, volume~3, pages 887--903. Elsevier,
  2003.

\bibitem[Jeremy(2019)]{ReizensteinIhesis2018}
Reizenstein Jeremy.
\newblock \emph{Iterated-Integral Signatures in Machine Learning}.
\newblock PhD thesis, 2019.

\bibitem[Ni(2015)]{ni2015multi}
Hao Ni.
\newblock A multi-dimensional stream and its signature representation.
\newblock \emph{arXiv preprint arXiv:1509.03346}, 2015.

\end{thebibliography}

\end{document}